\documentclass{article}

\usepackage[final]{neurips_2025}

\usepackage[utf8]{inputenc} % allow utf-8 input
\usepackage[T1]{fontenc}    % use 8-bit T1 fonts
\usepackage{hyperref}       % hyperlinks
\usepackage{url}            % simple URL typesetting
\usepackage{booktabs}       % professional-quality tables
\usepackage{amsfonts}       % blackboard math symbols
\usepackage{nicefrac}       % compact symbols for 1/2, etc.
\usepackage{microtype}      % microtypography
\usepackage{xcolor}         % colors
\usepackage{natbib}
\usepackage{titletoc}
\usepackage{xspace}
\usepackage{amsmath} 
\usepackage{multicol}
\usepackage{graphicx}
\usepackage{subcaption}
\usepackage{float}
\usepackage{multirow}
\usepackage{array}
\usepackage[table]{xcolor}
\usepackage{multirow}
\usepackage{bbm}
\usepackage{fontawesome}

\usepackage{pifont} % For special symbols like hearts and diamonds

\usepackage{tcolorbox}
\usepackage{titlesec}

\usepackage{wrapfig}
\usepackage{enumitem}
\usepackage{siunitx}     % 数字对齐
\usepackage{amsmath,amsfonts,bm}
\usepackage{amssymb} % 提供 \checkmark 命令
\usepackage{textcomp}

\usepackage{amsthm}
\usepackage{mathrsfs}
\theoremstyle{plain}
\newtheorem{theorem}{Theorem}

\newcommand{\cL}{\mathcal{L}}         % Calligraphic L for loss
\newcommand{\student}{\mathsf{stu}}
\newcommand{\teacher}{\mathsf{tea}}
\newcommand{\TV}{\mathrm{TV}}
\newcommand{\PreserveBackslash}[1]{\let\temp=\\#1\let\\=\temp}
\newcolumntype{C}[1]{>{\PreserveBackslash\centering}p{#1}}
\newcolumntype{R}[1]{>{\PreserveBackslash\raggedleft}p{#1}}
\newcolumntype{L}[1]{>{\PreserveBackslash\raggedright}p{#1}}

\definecolor{mygray}{gray}{.92}
\definecolor{ForestGreen}{RGB}{34,139,34}
\newcommand{\fg}[1]{\mathbf{\mathcolor{ForestGreen}{#1}}}
\definecolor{Forestred}{RGB}{220,50,50}

\definecolor{optimum}{RGB}{70,52,128}

\definecolor{darkpink}{RGB}{255, 20, 147}

\newcommand{\algname}{EPIC\xspace}

\title{Efficient Multi-modal Large Language Models via Progressive Consistency Distillation}

\author{
    {\bf Zichen Wen}$^{1, 2}$ \quad 
    {\bf Shaobo Wang}$^{1}$ \quad
    {\bf Yufa Zhou}$^{3}$ \quad
    {\bf Junyuan Zhang}$^{4}$ \quad
    {\bf Qintong Zhang}$^{5}$ \\
    {\bf Yifeng Gao}$^1$ \quad 
    {\bf Zhaorun Chen}$^6$ \quad
    {\bf Bin Wang}$^2$ \quad
    {\bf Weijia Li}$^{7, 2}$ \quad \\
    {\bf Conghui He}$^{2}$\footnotemark[1] \quad
    {\bf Linfeng Zhang}$^{1}$\footnotemark[1] \\
    \textsuperscript{1}EPIC Lab, Shanghai Jiao Tong University 
    \quad
    \textsuperscript{2}Shanghai AI Laboratory  \\
    \textsuperscript{3}Duke University     
    \quad
    \textsuperscript{4}The University of Hong Kong \\
    \textsuperscript{5}Peking University      
    \quad
    \textsuperscript{6}University of Chicago
    \quad
    \textsuperscript{7}Sun Yat-sen University  \\
    \normalsize
    \texttt{heconghui@pjlab.org.cn, zhanglinfeng@sjtu.edu.cn}
    \vspace{4pt} \\
    \faGithub \, \textbf{Code:} \href{https://github.com/ZichenWen1/EPIC}{\textcolor{darkpink}{https://github.com/ZichenWen1/EPIC}}
}

\begin{document}

\maketitle
{
\renewcommand{\thefootnote}{\fnsymbol{footnote}}
\footnotetext[1]{Corresponding authors.}
}

% \raisebox{-0.3\height}{\hspace{-0.10cm}\includegraphics[width=0.45cm]{figure/git_logo.png}} \small \textbf{\mbox{Code Repository:}} \href{https://github.com/ZichenWen1/EPIC}{github.com/ZichenWen1/EPIC}

\begin{abstract}
Visual tokens consume substantial computational resources in multi-modal large models (MLLMs), significantly compromising their efficiency.
Recent works have attempted to improve efficiency by compressing visual tokens during training, either through modifications to model components or by introducing additional parameters. However, they often overlook the increased learning difficulty caused by such compression, as the model’s parameter space struggles to quickly adapt to the substantial perturbations in the feature space induced by token compression.
In this work, we propose to develop \textbf{E}fficient MLLMs via \textbf{P}rogress\textbf{I}ve \textbf{C}onsistency Distillation (\texttt{\algname}), a progressive learning framework. Specifically, by decomposing the feature space perturbations introduced by token compression along the token-wise and layer-wise dimensions, we introduce token consistency distillation and layer consistency distillation, respectively, aiming to reduce the training difficulty by leveraging guidance from a teacher model and following a %curriculum-style, easy-to-hard 
progressive learning trajectory.
% 是否还需要更详细的描述？
Extensive experiments demonstrate the superior effectiveness, robustness, and generalization capabilities of our proposed framework.

% \raisebox{-0.3\height}{\includegraphics[width=0.4cm]{figure/hf_logo.png}} \small \textbf{\mbox{Models:}} \href{xxx}{xxx} \\
% \vspace{1em}

\end{abstract}

\section{Introduction}
\label{sec:intro}
Multi-modal large language models (MLLMs)~\citep{li2023blip, zhu2023minigpt,liu2024improved,liu2024visual} equip large language models (LLMs)~\citep{touvron2023llama, grattafiori2024llama} with the ability to understand visual information, exhibiting remarkable capabilities across a diverse range of multi-modal tasks, including image captioning, visual question answering (VQA), video understanding~\citep{wang2024internvideo2}, and multi-modal reasoning~\citep{wang2024exploring}. 

However, unlike LLMs~\citep{touvron2023llama, bai2023qwen, yang2024qwen2,wang2025data}, which only need to process a small number of information-dense text tokens~\citep{marr2010vision}, the introduction of substantial visual tokens in MLLMs~\citep{Bai:Qwen-VL, bai2025qwen2} presents significant computational challenges. 
This issue becomes particularly pronounced when handling high-resolution images~\citep{li2024mini} or multi-frame videos~\citep{tang2023video}. 
% For instance, in the widely adopted CLIP ViT-L/336px vision encoder, a single image is converted into $24 \times 24=576$ visual tokens~\citep{radford2021learning}, where incorporating such a substantial amount of tokens into the context of large-scale LLMs leads to considerable computational overhead and increased inference latency.

To address this issue, a natural approach is to reduce visual tokens~\citep{liu2025shifting, xiong2025prune2drive, chen2025variation, yang2025efficientvla}. Recent advances have introduced various token compression techniques to eliminate vision tokens in a training-free approach~\citep{zhang2024sparsevlm,wen2025token}.
Most of them adopt either a token importance-based strategy~\citep{chen2024image, he2024zipvl} 
%(selecting crucial visual tokens by defining importance scores)~\citep{chen2024image, he2024zipvl} 
or a token redundancy-based strategy~\citep{wang2025folder,tan2025tokencarve}. %(removing redundant visual tokens by measuring token similarity)~\citep{wang2025folder,tan2025tokencarve}. 
Although the aforementioned non-parametric methods avoid additional training costs, they inevitably incur significant performance degradation.
To achieve a better performance-efficiency trade-off, recent training-aware token compression methods have attracted significant attention~\citep{cha2024honeybee, chu2024mobilevlm}.
Beyond early approaches~\cite{li2023blip} that employed Q-former, MQT-LLaVA~\citep{hu2024matryoshka} proposes a more flexible Q-former capable of dynamically encoding visual information into variable-length visual tokens.
% In addition to Q-Former, which serves as a projector for token compression, 
\citep{li2024tokenpacker} proposes TokenPacker, a coarse-to-fine visual projector that progressively refines coarse tokens with visual information.
% Beyond architectural modifications and enhancements to model components, 
In addition to refining the model architecture and upgrading model components,
VoCo-LLaMA~\citep{ye2024voco} and LLaVA-Mini~\citep{zhang2025llava} build upon observations of the transfer of visual information within language models, attempting to transfer visual information into a small set of VoCo tokens via attention modification, or into textual tokens with a transformer-based pre-fusion module.
\begin{figure*}[!t]
    \centering
    \includegraphics[width=1.0\linewidth]{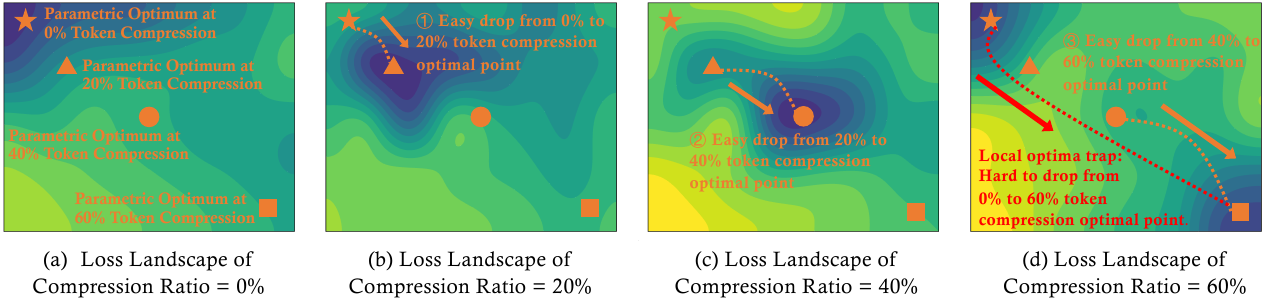}
    \vspace{-5mm}
    \caption{Progressive Consistency Distillation vs. Direct Training. Each subplot shows the loss landscape under the corresponding token compression ratio, with the \textbf{\textcolor{optimum}{optimum}} indicated. Our method reaches the objective via \textbf{\textcolor{orange}{progressive}} learning trajectories, while \textbf{\textcolor{red}{direct}} training remains challenging.}
    \label{fig:motivation}
    \vspace{-7mm}
\end{figure*}
% Progressive Consistency Distillation vs. Direct Training. Our method achieves the final objective through \textcolor{orange}{progressive} learning trajectories, whereas \textcolor{red}{direct} token compression training proves exceedingly challenging.
% Although these methods have achieved improved model efficiency while largely preserving performance, their contributions mainly stem from the introduction of new modules or upgrades to model components, 
% They tend to overlook the exploration of training methodologies.
% tending to overlook the challenges associated with token compression during training.
%, as well as the exploration of corresponding training methods and strategies. \\
% --------
% Naturally, directly applying token compression during training is challenging, as it introduces significant perturbations in the feature space, making it difficult for the parameter space to adapt to such drastic changes through supervised fine-tuning.

Although these methods improve model efficiency while largely preserving performance, their improvements primarily come from architectural enhancements or newly introduced modules, often overlooking the learning challenges posed by token compression during training.
As illustrated in Figure~\ref{fig:motivation}, when token compression is applied during training, the distribution of the compressed token sequence inevitably differs from that of the full token set. This discrepancy can be regarded as a perturbation in the feature space, which shifts the model's optimal point in the parameter space. The higher the compression ratio, the greater the perturbation introduced, and consequently, the further the optimal point drifts. The goal of training-aware token compression is to guide the model to progressively adapt from the original optimum (under full tokens) to a new optimum corresponding to the compressed token distribution.
Figure~\ref{fig:motivation} (d) shows that directly training a model with compressed tokens often leads to suboptimal solutions, as the optimization process can easily get trapped in poor local minima, making it difficult to reach the desired optimum under heavy compression.
% \textbf{\emph{How to better train models by leveraging existing token compression techniques (such as training-free token pruning methods) without altering the model architecture remains an open area for further investigation.}}
% 渐进式学习的训练策略:
% 1. Token consistency distillation：具体实施是：该方案在固定一层进行token pruning，teacher model 和 student model是同一个model，即共享权重，teacher model和 student model的区别在于token pruning ratio的不同，teacher model的pruning ratio更小(teacher model的pruning ratio = student_model_pruning_ratio - gap)，从teacher model到student model做distill，在训练期间，student model的采样的pruning ratio的上边界是由小到大的（对应训练难度从下到大），gap between teacher and student 也是由小到大（不过上限设定为0.3）
% 2. Layer consistency distillation：具体实施是：该方案的pruning layer是从深层（第20层）到浅层（第2层），这个时候student model的pruning ratio不再是从小到大的，而是在（0.2， 0.9）范围内均匀采样，student model 和 teacher model之间仍然存在一个pruning ratio的gap,随着训练进行，pruning layer逐渐过渡到浅层（在浅层pruning 训练难度比深层要大）。

% 两种方案都蕴含课程学习（Curriculum Learning）：训练由简单到困难
%（1） token consistency -> pruning ratio 由小到大；gap between teacher and student 由小到大
%（2） layer consistency -> layer 由深到浅 (在deeper layer token pruning 更加容易)

\begin{wrapfigure}{r}{0.42\textwidth} 
\vspace{-4mm}
    \centering
    \includegraphics[width=\linewidth]{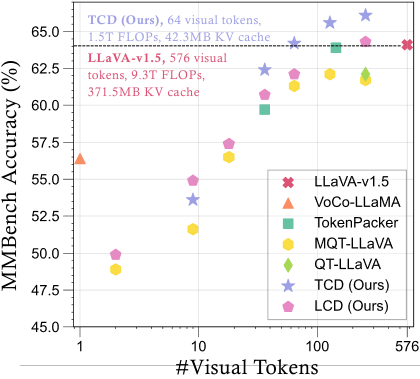}
    \vspace{-4mm}
    \captionsetup{width=0.42\textwidth}
    \caption{MMBench accuracy vs. number of visual tokens for various methods. TCD (Ours) and LCD (Ours) achieve competitive accuracy with far fewer tokens, lower FLOPs, and smaller KV cache compared to LLaVA-v1.5, highlighting its efficiency.}
    \label{fig:ill}
    \vspace{-5mm}
\end{wrapfigure}

In this work, 
%inspired by curriculum learning~\citep{bengio2009curriculum,wang2021survey} and Vygotsky's zone of proximal development (ZPD)~\citep{fani2011implications}, 
we propose a progressive consistency 
distillation learning framework \texttt{\algname} tailored for token compression, where a single MLLM simultaneously acts as both teacher and student through weight sharing.
From a token-wise perspective, we introduce \textbf{Token Consistency Distillation (TCD)}. At the early stages of training, both the teacher and student models adopt a very low token compression ratio, indicating a relatively easy learning task without significant optimal point drifts (as shown in Fig.~\ref{fig:motivation} (b)). As training progresses, the compression becomes increasingly aggressive, forming a progressive learning trajectory (see Fig.~\ref{fig:motivation} (b) to (d)). Although the final optimal point under heavy compression remains far from the initial optimum, each intermediate optimum along the trajectory is relatively close to its predecessor, making each transition more manageable and easier to optimize.
Moreover, the teacher consistently uses a slightly lower compression ratio than the student, introducing a compression ratio gap between them. We argue that when the gap is too large, the student may struggle to benefit effectively from the teacher’s guidance~\citep{fani2011implications}. Therefore, the compression ratio gap is also designed to follow a progressive learning strategy, gradually increasing over time to ease the learning process.
From a layer-wise perspective, we introduce \textbf{Layer Consistency Distillation (LCD)}.
Based on observations from prior work~\citep{chen2024image}, the significance of visual tokens diminishes notably in deeper layers, suggesting that compressing tokens at these layers has minimal impact on the model's feature space and output. 
Therefore, in LCD, the token compression is progressively shifted from deeper to shallower layers as training progresses, implicitly following an easy-to-hard learning paradigm. Meanwhile, a compression gap between the teacher and the student is maintained to encourage effective guidance from the teacher.

Our primary contribution lies in proposing \texttt{\algname}, a progressive consistency distillation learning framework, which demonstrates compatibility with diverse token compression techniques. Within this framework, we introduce Token Consistency Distillation and Layer Consistency Distillation from token-wise and layer-wise dimensions, respectively. This approach enables the training of robust and highly generalizable models through a progressive learning strategy without requiring modifications to the model architecture. Compared to prior approaches, comprehensive experiments validate the superior effectiveness, robustness, and generalization capabilities of our proposed framework.

% In summary, we make the following key contributions: 

\section{Related Works}
% \subsection{Multi-modal Large Language Models}
\vspace{-2mm}
\paragraph{Multi-modal Large Language Models.}
Multi-modal Large Language Models (MLLMs), pioneered by~\citep{liu2023llava, zhu2023minigpt, instruct-blip, yin2023survey,chen2024mj} have successfully showcased promising results on a wide variety of vision-language perception and reasoning tasks. 
Existing MLLMs typically employ a pre-trained vision encoder (\emph{e.g.}, CLIP~\citep{radford2021learning} and SigLIP~\citep{zhai2023sigmoid, tschannen2025siglip}) to extract visual features, which are then projected into the LLM's input space via a visual projector (\emph{e.g.}, MLP and Q-former~\citep{li2023blip}), enabling the model to process both visual embeddings and user instructions for multimodal understanding and response generation.
Recent studies~\citep{zhang2023visual, zhang2024exploring} have highlighted the limitations of MLLMs in fine-grained visual perception tasks. 
To this end, more advanced MLLMs have attempted to increase the number of encoded visual tokens by employing dynamic resolutions~\citep{chen2024far, chen2024expanding, liu2024llavanext, li2024llava} or arbitrary resolutions~\citep{wang2024qwen2, bai2025qwen2} to process high-resolution images, thereby enhancing their performance in visual understanding tasks. However, due to the quadratic complexity of the attention mechanism~\citep{vaswani2017attention}, the resulting longer token sequences pose significant challenges to both inference speed and memory usage.

% \subsection{Visual Token Compression}
\vspace{-2mm}
\paragraph{Visual Token Compression.}
Visual tokens typically outnumber text tokens by orders of magnitude while containing greater spatial redundancy than information-dense text~\citep{marr2010vision}. Recent work has explored both training-free~\citep{he2024zipvl,zhang2024treat,zhao2024accelerating,liu2025compression,zhang2024beyond} and training-aware~\citep{cha2024honeybee,chu2023mobilevlm,Bai:Qwen-VL,chensafewatch,wang2025winningpruninggambleunified} compression approaches, with the latter showing superior performance potential despite requiring additional training~\citep{wen2025token}. Training-free methods generally follow two paradigms: importance-based strategies like FastV~\citep{chen2024image} and SparseVLM~\citep{zhang2024sparsevlm} that use attention scores, and redundancy-based approaches such as DART~\citep{wen2025stop} and G-Prune~\citep{jiang2025kind} that assess token similarity.
Early training-aware work focused on parameter-free strategies, including LLaVA-PruMerge's attention-based merging~\citep{Shang2024:LLaVA-PruMerge} and VoCo-LLaMA's compressed VoCo tokens~\citep{ye2024voco}. Subsequent research explored architectural modifications through lightweight component replacements~\citep{instructblip,dong2024internlm,li2024llama}. More advanced approaches include MQT-LLaVA's dynamic Q-former for variable-length token encoding~\citep{hu2024matryoshka} and TokenPacker's coarse-to-fine iterative condensation~\citep{li2024tokenpacker}. Recent methods like LLaVA-Mini~\citep{zhang2025llava} achieve near-lossless token compression through extra auxiliary modules, though they often overlook the training challenges introduced by feature space perturbations during token compression.

\section{Methodology}
\label{sec:method}
\vspace{-2mm}
\subsection{Preliminary}
\label{sec:preliminary}

\vspace{-2mm}
\paragraph{Multi-modal Large Language Models (MLLMs).}
% An MLLM typically comprises three modules: a \textit{visual encoder} $\mathrm{VE}$, a \textit{modality projector} (\emph{e.g.}, a learnable MLP), and a \textit{language model} $\mathrm{LM}$. Given an input image $\mathcal{I}$, the visual encoder extracts patch-level features $\mathbf{z}_v \in \mathbb{R}^{N \times d_v}$, which are projected into visual tokens $\mathbf{e}_v \in \mathbb{R}^{N \times d_h}$:
An MLLM typically consists of three modules: a \textit{visual encoder} $\mathrm{VE}$, a \textit{modality projector} (\emph{e.g.}, MLP), and a \textit{language model} $\mathrm{LM}$. Given image $\mathcal{I}$, the visual encoder extracts patch-level features $\mathbf{z}_v \in \mathbb{R}^{N \times d_v}$, projected into visual tokens $\mathbf{e}_v \in \mathbb{R}^{N \times d_h}$:
\begin{equation}
    \mathbf{e}_v = \mathrm{MLP}(\mathrm{VE}(\mathcal{I})),
    \label{eq:visual_proj}
\end{equation}
where $N$ is the number of image patches, $d_v$ is the encoder output dimension, and $d_h$ is the LM hidden size.
Meanwhile, a text prompt $\mathcal{P}$ is tokenized and embedded into a sequence of text tokens $\mathbf{e}_t \in \mathbb{R}^{L \times d_h}$.
The visual and text tokens are then concatenated to form the full input sequence:
\begin{equation}
    \mathbf{x} = [\mathbf{e}_v; \mathbf{e}_t] \in \mathbb{R}^{(N + L) \times d_h} .
\end{equation}
Positional embeddings are added to $\mathbf{x}$ to encode spatial and sequential structure.
The language model then autoregressively generates output tokens \( y_i \in \mathcal{V} \) (with vocabulary \( \mathcal{V} \)) one token at a time:
\begin{equation}
    y_i = \mathrm{LM}(\mathbf{x}, y_{<i}), \quad \text{for } i = 0, 1, 2, \dots
    \label{eq:lm_generation}
\end{equation}
where $y_{<i} := \{y_0, y_1, \dots, y_{i-1}\}$ denotes the previously generated tokens.

\begin{figure*}[!t]
    \centering
    \vspace{-2mm}
    \includegraphics[width=1.0\linewidth, height=0.4\linewidth]{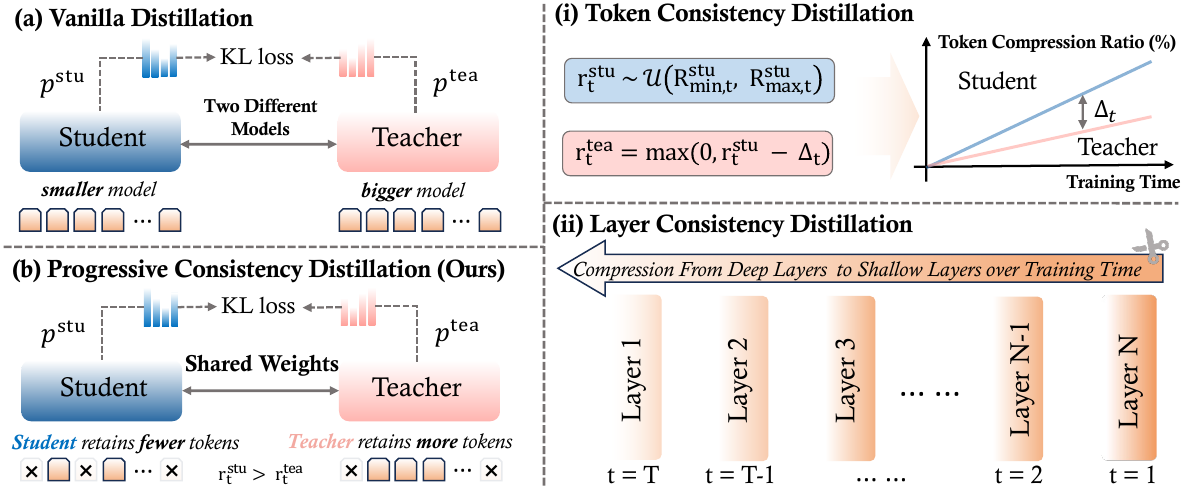}
    \caption{\textbf{An overview of Progressive Consistency Distillation.} (i) Token Consistency Distillation progressively increases token compression ratio over time. (ii) Layer Consistency Distillation shifts token compression from deep to shallow layers, promoting layer-wise consistency during training.}
    \vspace{-5.5mm}
    \label{fig:framework}
\end{figure*}

\vspace{-2mm}
\paragraph{Progressive Consistency Distillation Learning.}
Training MLLMs under aggressive token compression introduces significant feature space perturbations that hinder convergence. To address this, we propose \emph{Progressive Consistency Distillation Learning} (PCDL), %which gradually increases compression difficulty through complementary temporal and spatial strategies. 
which gradually increases compression difficulty and facilitates the model's progressive convergence to the final objective (see Fig.~\ref{fig:motivation}).
Our \textbf{Token Consistency Distillation (TCD)} progressively increases the token compression ratio over training steps from a token-wise perspective, while \textbf{Layer Consistency Distillation (LCD)} initially applies token compression in deeper layers (minimal impact~\citep{chen2024image,zhang2025llava}) before gradually shifting to shallower layers, performing progressive learning in a layer-wise manner. Both employ \emph{Consistency Distillation with Shared Weights}, where teachers with slightly lower compression ratios (\emph{e.g.}, $5\%$ less) provide manageable guidance~\citep{fani2011implications}, progressively transitioning to stronger teachers (\emph{e.g.}, $10\%$ lower ratio) for staged mentorship.

\vspace{-2mm}
\subsection{Theoretical Intuition: A 1D Prototype for Progressive Consistency Distillation Learning}
\label{sec:scalar-proto}
\vspace{-2mm}
\paragraph{Scalar center path.}
To provide intuition for \textit{Progressive Consistency Distillation Learning} (PCDL), we introduce a one-dimensional prototype where model predictions are scalar values \(\theta \in \mathbb{R}\), and each target is given by a compression-dependent center \(c_r\in\mathbb{R}\) for compression ratio \(r \in [0, r_{\max}]\). We assume that the mapping \(c : [0, r_{\max}] \to \mathbb{R}\) satisfies:
\vspace{-2mm}
\begin{enumerate}[leftmargin=*, topsep=2pt, itemsep=2pt, parsep=0pt, label=\textbf{(S\arabic*)}]
  \item \emph{Monotonicity:} \(c\) is differentiable with \(c'_r \ge 0\) for all \(r\).
  \item \emph{Lipschitz slope:} There exists \(\gamma > 0\) such that \(|c'_r| \le \gamma\), and \(|c_{r_1} - c_{r_2}| \le \gamma |r_1 - r_2|\) for all \(r_1, r_2\).
  \item \emph{Convexity:} \(c\) is convex, i.e., \(c''(r) \ge 0\).
\end{enumerate}
% We use \(c_r\) and \(c(r)\) interchangeably.
\vspace{-3mm}
\paragraph{Two quadratic objectives.}
We compare two learning objectives:
% \begin{align*}
%   \cL_{\mathrm{dir}}(r, \theta)
%   &= \tfrac{1}{2} (\theta - c_r)^2, \\
%   \cL_{\mathrm{prog}}(r, \theta)
%   &= \tfrac{1}{2} (\theta - c_r)^2
%    + \tfrac{\lambda}{2} (\theta - c_{r - \Delta})^2,
% \end{align*}
\[
\cL_{\mathrm{dir}}(r, \theta) = \tfrac{1}{2} (\theta - c_r)^2,\qquad 
\cL_{\mathrm{prog}}(r, \theta) = \tfrac{1}{2} (\theta - c_r)^2 + \tfrac{\lambda}{2} (\theta - c_{r - \Delta})^2,
\]
with constants \(0 < \lambda < 1\), \(0 < \Delta \le r_{\max}\). The second term acts as a regularizer pulls the prediction toward a slightly less compressed teacher target, mimicking the KL-based distillation loss in \texttt{\algname}.

\paragraph{Analogy to \algname.}
In \texttt{\algname}, the teacher operates at a slightly lower compression ratio \((r - \Delta)\) and provides smoother targets via KL consistency. The scalar prototype captures this: %relationship: 
\(\cL_{\mathrm{dir}}\) corresponds to %naive 
direct supervision from the current target \(c_r\), while \(\cL_{\mathrm{prog}}\) encodes progressive distillation by incorporating a past (easier) target \(c_{r-\Delta}\). As in \texttt{\algname}, the regularizer improves training stability by reducing sensitivity to abrupt changes in the input space induced by token compression.

\paragraph{Exact minimizers.}
Both objectives are strongly convex in \(\theta\) with closed-form solutions:
\[
  \theta^{\mathrm{dir}}_r = c_r,
  \qquad
  \theta^{\mathrm{prog}}_r = \frac{c_r + \lambda\,c_{r - \Delta}}{1 + \lambda}.
\]

\paragraph{Schedule and path length.}
Let \(0 = r_0 < r_1 < \dots < r_T = r_{\max}\) denote any compression schedule. The total variation of a path \(\{x_t\}\) is defined as:
\[
  \TV(\{x_t\}) := \sum_{t=0}^{T-1} |x_{t+1} - x_t|.
\]
We show below that the progressive path yields strictly shorter total variation, reflecting the smoother optimization trajectory induced by PCDL. We delay the full proof to Appendix~\ref{app_sec:proof}.

\begin{theorem}[Scalar path gain, Proof in Appendix~\ref{app_sec:proof}]
\label{thm:scalar}
Under assumptions \textbf{(S1)}–\textbf{(S3)}, the total variation of the progressive path is strictly smaller:
\begin{align*}
    & \TV\bigl(\{\theta^{\mathrm{dir}}_{r_t}\}\bigr) \le \gamma r_{\max}, \\
    & \TV\bigl(\{\theta^{\mathrm{prog}}_{r_t}\}\bigr) 
    \le \frac{1 + \lambda \kappa}{1 + \lambda} \cdot \TV\bigl(\{\theta^{\mathrm{dir}}_{r_t}\}\bigr)
    \;<\; \TV\bigl(\{\theta^{\mathrm{dir}}_{r_t}\}\bigr),
\end{align*}
% where
\[
  \text{where} \qquad \kappa := \sup_{r \in [\Delta,\, r_{\max} - \Delta]}
  \frac{c(r) - c(r - \Delta)}{c(r + \Delta) - c(r)} \in [0, 1).
\]
\end{theorem}

\subsection{Token Consistency Distillation}
\vspace{-2mm}
We consider a single MLLM $f_\theta$ that plays both teacher and student roles by sharing parameters $\theta$.  
Let $\mathcal{I}$ denote the input image and $\mathcal{P}$ denote the language prompt.  
Let $C(\mathcal{I}, r, \ell)$ be a token compression operator\footnote{Any plug-and-play token compressor (\emph{e.g.}, FastV~\citep{chen2024image}, DART~\citep{wen2025stop}) can serve as compression operator here.} that, at layer $\ell$, compresses the visual tokens extracted from $\mathcal{I}$ with ratio $r \in [0, 1]$, \emph{i.e.}, retaining only a fraction $1 - r$ of visual tokens. Let $t\in\{0,\dots,T\}$ index training steps.

At iteration $t$, we sample a student compression ratio from a gradually shifting uniform distribution:
\begin{equation}
    r^{\student}_t \sim \mathcal{U}\bigl(R^{\student}_{\text{min},t},\; R^{\student}_{\text{max},t}\bigr),
\end{equation}
where $R^{\student}_{\text{max},t}$ linearly increases from $R^{\student}_{\text{max},0} = \epsilon$ (\emph{e.g.}, $5\%$) to $R^{\student}_{\text{max},T} = R_{\max}$ (\emph{e.g.}, $90\%$) with training steps, and $R^{\student}_{\text{min},t}$ grows more slowly from $0\%$ to at most $50\%$ as training progresses. This creates an easy-to-hard curriculum for the student model by gradually narrowing and shifting the sampling range towards more aggressive compression.

We then define the teacher model’s compression ratio as
\begin{equation}
    r^{\teacher}_t = \max\bigl(0,\; r^{\student}_t - \Delta_t \bigr),
\end{equation}
where the compression gap $\Delta_t$ also increases gradually from $\Delta_0 = \delta_{\min}$ to $\Delta_T = \delta_{\max}$ (\emph{e.g.}, up to $30\%$), ensuring that the teacher consistently sees slightly less compressed inputs than the student.

Concretely, we form two forward passes through the shared model $f_\theta$:  
% \begin{align}
%     \mathbf{h}^{\teacher} &= f_\theta\bigl(C(\mathcal{I}, r^{\teacher}_t, \ell),\; \mathcal{P} \bigr), \\
%     \mathbf{h}^{\student} &= f_\theta\bigl(C(\mathcal{I}, r^{\student}_t, \ell),\; \mathcal{P} \bigr),
% \end{align}
\begin{equation}
    \mathbf{h}^{\teacher} = f_\theta\bigl(C(\mathcal{I}, r^{\teacher}_t, \ell); \mathcal{P} \bigr); \quad
    \mathbf{h}^{\student} = f_\theta\bigl(C(\mathcal{I}, r^{\student}_t, \ell); \mathcal{P} \bigr),
\end{equation}
where both token compressions occur at the fixed Transformer layer $\ell$.

We incorporate the Token Consistency Distillation (TCD) loss as an auxiliary objective alongside the main supervised fine-tuning (SFT) loss. Specifically, the overall training objective is formulated as:
\begin{equation}
    \mathcal{L}_{\text{total}}(\theta) = (1 - \lambda) \cdot \mathcal{L}_{\text{SFT}}(\theta) + \lambda \cdot \mathcal{L}_{\text{TCD}}(\theta),
\end{equation}
where $\mathcal{L}_{\text{SFT}}(\theta)$ denotes the autoregressive cross-entropy loss over language outputs, and $\lambda$ is a balancing coefficient for the distillation loss, empirically set to $0.7$.

The TCD loss is defined as the Kullback–Leibler (KL) divergence~\citep{kullback1951information,hinton2015distilling} between the output logits distributions of the teacher and student:
\begin{equation}
    \mathcal{L}_{\mathrm{TCD}}(\theta)
    = \mathbb{E}_{\mathcal{I},\mathcal{P},t}\left[\mathrm{KL}\left(p^{\teacher} \;\|\; p^{\student}\right)\right],
\end{equation}
where $p^{\teacher} = \mathrm{Softmax}(\mathbf{h}^{\teacher} / \tau)$ and $p^{\student} = \mathrm{Softmax}(\mathbf{h}^{\student} / \tau)$ are the temperature-scaled output logits distributions from the teacher and student, respectively, and $\tau$ is a temperature hyperparameter.

By progressively increasing both the student's compression range $\left[R^{\student}_{\text{min},t}, R^{\student}_{\text{max},t}\right]$ and the teacher–student gap $\Delta_t$, we implement an progressive learning in a token-wise manner. At early training stages, the student experiences mild compression and benefits from a closely aligned teacher, while in later stages, it endures stronger compression with increasingly distinct teacher guidance.

\subsection{Layer Consistency Distillation}
\vspace{-2mm}
% ------------- v1 --------------
% In addition to token consistency distillation, building upon observations from prior work~\citep{chen2024image, zhang2025llava}, we note that for the visual modality, the shallow layers of language models play a more critical role than the deeper layers, where visual tokens are assigned negligible attention weights. We argue that performing token compression at deeper layers introduces minimal perturbation to the output while causing less distortion in the feature space, which motivates our proposed another progressive learning strategy \textbf{Layer Consistency Distillation (LCD)} in spatial dimension.

% -------------v2 ---------------
Prior work~\citep{chen2024image,zhang2025llava} shows visual tokens receive negligible attention in deeper layers while shallow layers play a more critical role for visual modality. This motivates our \textbf{Layer Consistency Distillation (LCD)} strategy: performing compression in deeper layers first minimizes output perturbation and feature space distortion, then progressively moves to shallower layers.

Let $L$ be the total number of Transformer layers in the language model. Define a normalized training progress: $\beta_t = \frac{t}{T}$.
We then select a single compression layer for both teacher and student:  
% \begin{equation}
%     \ell_t = \bigl\lceil\,L - \beta_t\,(L - \ell_{min})\bigr\rceil,
%     \quad \ell_t\in\{\ell_{min},\dots,L\},
% \end{equation}
\begin{equation}
\ell_t = \mathrm{Round}\bigl(L - \beta_t (L - \ell_{\min})\bigr),
\quad \ell_t \in {\ell_{\min}, \ell_{\min}+1, \dots, L}
\end{equation}
where $\ell_{\min}$ denotes the shallowest compression layer.
Thus at $t=0$ we compress at the deepest layer $\ell_0 = L$, and by $t=T$ at the shallowest layer $\ell_T = \ell_{\min}$, realizing a layer-wise progressive learning.

At training iteration $t$, we sample the student model’s token compression ratio from $[r_{\min},r_{\max}]$:
\begin{equation}
    r^{\student}_t \sim \mathcal{U}(r_{\min},\,r_{\max}),
\end{equation}
Typically, $r_{\min}$ is set to $0.2$, while $r_{\max}$ is configured as $0.9$ and define the teacher model’s 
compression ratio by subtracting the compression ratio gap $\Delta_t$:
\begin{equation}
    r^{\teacher}_t = \max\bigl(0,\;r^{\student}_t - \Delta_t \bigr).
\end{equation}
We then perform two forward passes through the shared model $f_\theta$ at layer $\ell_t$:
% \begin{align}
%     \mathbf{h}^{\teacher} &= f_\theta\bigl(C(\mathcal{I},r^{\teacher}_t,\ell_t),\,\mathcal{P}\bigr),\\
%     \mathbf{h}^{\student} &= f_\theta\bigl(C(\mathcal{I},r^{\student}_t,\ell_t),\,\mathcal{P}\bigr).
% \end{align}
\begin{equation}
\mathbf{h}^{\teacher} = f_\theta\bigl(C(\mathcal{I},r^{\teacher}_t,\ell_t);\mathcal{P}\bigr); \quad
\mathbf{h}^{\student} = f_\theta\bigl(C(\mathcal{I},r^{\student}_t,\ell_t);\mathcal{P}\bigr).
\end{equation}
The Layer Consistency Distillation (LCD) loss is defined analogously to TCD, as the KL divergence between the teacher and student output logits distributions:
\begin{equation}
    \mathcal{L}_{\mathrm{LCD}}(\theta)
    = \mathbb{E}_{\mathcal{I},\mathcal{P},t}\!\bigl[\mathrm{KL}(p^{\teacher} \,\|\, p^{\student})\bigr];
    \quad
    p^{\teacher} = \mathrm{Softmax}(\mathbf{h}^{\teacher}/\tau),
    p^{\student} = \mathrm{Softmax}(\mathbf{h}^{\student}/\tau).
\end{equation}

Finally, we integrate both distillation terms into the overall training objective:
\begin{equation}
    \mathcal{L}_{\mathrm{total}}(\theta)
    = (1-\lambda) \cdot \mathcal{L}_{\mathrm{SFT}}(\theta)
      + \lambda \cdot \mathcal{L}_{\mathrm{LCD}}(\theta).
\end{equation}
% where $\lambda$ (\emph{e.g.}, $0.7$) weights each distillation term equally against the supervised fine-tuning loss.  

\section{Experiments}
\label{sec:exp}
% \vspace{-2mm}
\subsection{Experimental Setting}
\label{sec:experimental_setting}

% \vspace{-1mm}
\paragraph{Implementation Details.}
\label{sec:implementation}
We implement \texttt{\algname} based on LLaVA~\citep{liu2023llava, liu2024llavanext} without introducing any modifications to the model architecture.
Specifically, we adopt CLIP ViT-L/14~\citep{radford2021learning} as our vision encoder, 
%supporting a 336 $\times$ 336 image resolution 
utilizing its officially pretrained projector, and Vicuna-v1.5~\citep{Vicuna} as our LLM.
\texttt{algname} only requires performing the second stage training, which involves visual instruction tuning on the LLaVA-665K instruction fine-tuning dataset. 
To demonstrate the effectiveness and generalizability of our method, we incorporate three representative token compression techniques (DART~\citep{wen2025stop}, FastV~\citep{chen2024image}, and Random token pruning) into \texttt{\algname} for training.
More implementation details about our proposed method and baselines are provided in Appendix~\ref{app_sec:train_detail}.

% \vspace{-2mm}
\paragraph{Evaluation Benchmarks.} %We evaluate our model across 10 representative visual understanding benchmarks, including VizWiz~\citep{gurari2018vizwiz}, ScienceQA-IMG~\citep{lu2022learn}, VQA-V2~\citep{balanced_vqa_v2}, TextVQA~\citep{singh2019towards}, OCRBench~\citep{liu2024ocrbench}, GQA~\citep{hudson2019gqa}, POPE~\citep{li-etal-2023-evaluating}, MME~\citep{fu2023mme}, MMBench~\citep{liu2023mmbench}, MMBench-CN~\citep{liu2023mmbench}.
% We adopt the LMMs-Eval~\citep{zhang2024lmmsevalrealitycheckevaluation} platform to evaluate performance across all benchmarks. Further details can be found in the Appendix~\ref{app_sec:benchmarks}.
We evaluate our model across 10 representative visual understanding benchmarks. Further details about benchmarks can be found in the Appendix~\ref{app_sec:benchmarks}.

% \vspace{-2mm}
\paragraph{Baselines.}
As shown in Table~\ref{table:main_result}, we compare our method with various visual token compression techniques, including QT-LLaVA%\footnote{The LLaVA baseline uses a query transformer with fixed 256 visual tokens.}
, MQT-LLaVA~\citep{hu2024matryoshka}, %PruMerge~\citep{Shang2024:LLaVA-PruMerge}, 
LLaMA-VID~\citep{li2024llama}, VoCo-LLaMA~\citep{ye2024voco}, TokenPacker~\citep{li2024tokenpacker}, and LLaVA-Mini~\citep{zhang2025llava}.
We also list other MLLM's results for comparison, including BLIP-2~\citep{li2023blip}, InstructBLIP~\citep{instruct-blip}, IDEFICS~\citep{laurenccon2023obelics}, Qwen-VL~\citep{Bai:Qwen-VL}, Qwen-VL-Chat~\citep{Bai:Qwen-VL}, SPHINX~\citep{lin2023sphinx}, mPLUG-Owl2~\citep{ye2024mplug}, and vanilla LLaVA~\citep{liu2023llava, liu2024llavanext}. Please refer to Appendix~\ref{app_sec:baselines} for more details.

% 还可以加llava next实验结果；加random，fastv相关的实验结果
\subsection{Experimental Results on Benchmarks}
\label{sec:main_results}
\definecolor{mygray}{gray}{.92}
\begin{table*}[!t]
\caption{Performance on 10 visual understanding benchmarks. ``Res.'' is resolution, and `\#Vision Tokens' is the number of vision tokens. Both training and inference employ DART as the token compression strategy for our methods. Parentheses in Avg.(\%) column show diffs vs. LLaVA-v1.5.}
% \vspace{1mm}
\label{table:main_result}
\centering
\tiny
\begin{tabular}%{L{1.56cm}L{1.12cm}c|c|C{0.27cm}C{0.27cm}C{0.27cm}C{0.27cm}C{0.27cm}C{0.27cm}C{0.38cm}C{0.27cm}C{0.27cm}C{0.38cm}|C{0.27cm}} 
{L{1.75cm}L{1.12cm}C{0.2cm}|C{0.55cm}|C{0.3cm}C{0.3cm}C{0.3cm}C{0.3cm}C{0.3cm}C{0.3cm}C{0.38cm}C{0.3cm}C{0.3cm}C{0.38cm}|C{0.84cm}} 
\toprule
\textbf{Methods} & \textbf{LLM} & \textbf{Res.} & \textbf{\begin{tabular}[c]{@{}c@{}}\#Vision\\ Tokens\end{tabular}} & $\!\!\textbf{\text{VQA}}^{\!\text{V2}}$ & $\!\!$\textbf{GQA} & $\!\!\!$\textbf{VizWiz} & $\!\!$$\textbf{\text{SQA}}^{\!\text{I}}$ & $\!\!$$\textbf{\text{VQA}}^{\!\text{T}}$ & $\!\!$\textbf{POPE} & \textbf{MME} & $\!\!$\textbf{MMB} & $\!$\textbf{\begin{tabular}[c]{@{}c@{}}MMB-\\ CN\end{tabular}} & $\!$\textbf{\begin{tabular}[c]{@{}c@{}}OCR\\ Bench\end{tabular}} & \textbf{\begin{tabular}[c]{@{}c@{}}Avg.\\ (\%)\end{tabular}} \\
\midrule
\textbf{BLIP-2}~\citep{li2023blip} & Vicuna-13B & 224 & 32 & 65.0 & 41.0 & 19.6 & 61.0 & 42.5 & 85.3 & -- & – & -- & -- & -- \\
\textbf{InstructBLIP}~\citep{instructblip} & Vicuna-7B & 224 & 32 & 66.3 & 49.2 & 34.5 & 60.5 & 50.1 & 83.9 & 1500 & 36.0 & -- & 259 & -- \\
\textbf{InstructBLIP}~\citep{instructblip} & Vicuna-13B & 224 & 32 & 64.2 & 49.5 & 33.4 & 63.1 & -- & 84.1 & 1530 & 36.9 & 17.4 & 252 & --  \\
\textbf{IDEFICS-9B}~\citep{laurenccon2023obelics} & LLaMA-7B & 224 & 64 & 50.9 & 38.4 & 35.5 & – & 25.9 & 75.3 & 1027 & 48.2 & -- & 245 & -- \\
\textbf{IDEFICS-80B}~\citep{laurenccon2023obelics} & LLaMA-65B & 224 & 64 & 60.0 & 45.2 & 36.0 & -- & 30.9 & -- & 1076 & 54.5 & 29.1 & 277 & -- \\
\textbf{Qwen-VL}~\citep{Bai:Qwen-VL} & Qwen-7B & 448 & 256 & -- & 59.3 & 35.2 & 67.1 & 63.8 & -- & 1708 & 38.2 & -- & 133 & -- \\
\textbf{Qwen-VL-Chat}~\citep{Bai:Qwen-VL} & Qwen-7B & 448 & 256 & -- & 57.5 & 38.9 & 68.2 & 61.5 & -- & 1891 & 60.6 & -- & 267 & -- \\
\textbf{SPHINX}~\citep{lin2023sphinx} & LLaMA-13B & 224 & 289 & 78.1 & 62.6 & 39.9 & 69.3 & 51.6 & 80.7 & -- & 66.9 & -- & -- & -- \\
\textbf{SPHINX-2k}~\citep{lin2023sphinx} & LLaMA-13B & 762 & 2890 & 80.7 & 63.1 & 44.9 & 70.6 & 61.2 & 87.2 & -- & 65.9 & -- & -- & -- \\
\textbf{mPLUG-Owl2}~\citep{ye2024mplug} & LLaMA-7B & 448 & 1024 & 79.4 & 56.1 & 54.5 & 68.7 & 54.3 & - & -- & 64.5 & -- & -- & -- \\
\textbf{Video-LLaVA}~\citep{lin2023video} & Vicuna-7B & 224 & 256 & 65.9 & 60.3 & 48.1 & 66.4 & 51.8 & 83.1 & 1542 & 60.6 & 49.3 & 161 & 55.7 \\
\textbf{LLaVA-v1.5}~\citep{liu2023llava} & Vicuna-7B & 336 & 576 & 72.2 & 61.9 & 52.5 & 68.3 & 58.1 & 85.9 & 1785 & 64.1 & 55.8 & 319 & 61.4 \\
% \textbf{LLaVA-v1.5} & Vicuna-13B & 336 & 576 & --  \\
% \textbf{LLaVA-v1.6} & -- \\
\rowcolor{mygray}\multicolumn{15}{c}{\scriptsize\textit{\textbf{LMMs with fewer vision tokens}}} \\
\textbf{Average-Pooling} & Vicuna-7B & 336 & 64 & 63.0 & 55.5 & 48.4 & 68.6 & 52.6 & 79.2 & 1579 & 59.6 & 49.5 & 258 & 55.9 {\color[HTML]{C2183A} \fontsize{5pt}{6pt}\selectfont (-5.5)}    \\
\textbf{MQT-LLaVA}~\citep{hu2024matryoshka} & Vicuna-7B & 336 & 2 & 51.4 & 49.6 & 50.0 & 66.1 & 14.8 & 75.4 & 1402 & 48.9 & 40.5 & 169 & 46.4 {\color[HTML]{C2183A} \fontsize{5pt}{6pt}\selectfont (-15)} \\
\textbf{MQT-LLaVA}~\citep{hu2024matryoshka} & Vicuna-7B & 336 & 36 & 62.0 & 57.7 & 53.6 & 69.2 & 28.6 & 82.9 & 1777 & 60.5 & 51.6 & 244 & 55.4 {\color[HTML]{C2183A} \fontsize{5pt}{6pt}\selectfont (-6.0)} \\
\textbf{MQT-LLaVA}~\citep{hu2024matryoshka} & Vicuna-7B & 336 & 64 & 65.6 & 58.7 & 54.3 & 68.4 & 32.5 & 83.1 & 1810 & 61.3 & 53.7 & 260 & 56.8 {\color[HTML]{C2183A} \fontsize{5pt}{6pt}\selectfont (-4.6)}  \\
\textbf{MQT-LLaVA}~\citep{hu2024matryoshka} & Vicuna-7B & 336 & 128 & 66.2 & 59.8 & 54.6 & 69.3 & 35.7 & 84.3 & 1773 & 62.1 & 53.6 & 266 & 57.6 {\color[HTML]{C2183A} \fontsize{5pt}{6pt}\selectfont (-3.8)} \\
\textbf{MQT-LLaVA}~\citep{hu2024matryoshka} & Vicuna-7B & 336 & 192 & 66.9 & 59.9 & 54.6 & 69.1 & 35.8 & 85.1 & 1784 & 62.0 & 53.9 & 263 & 57.7 {\color[HTML]{C2183A} \fontsize{5pt}{6pt}\selectfont (-3.7)} \\
\textbf{MQT-LLaVA}~\citep{hu2024matryoshka} & Vicuna-7B & 336 & 256 & 68.3 & 60.1 & 54.6 & 69.0 & 37.1 & 84.6 & 1740 & 61.7 & 53.0 & 273 & 57.8 {\color[HTML]{C2183A} \fontsize{5pt}{6pt}\selectfont (-3.6)} \\
\textbf{QT-LLaVA} & Vicuna-7B & 336 & 256 & -- & 60.3 & 51.5 & 68.1 & 36.9 & 84.1 & 1771 & 62.1 & 53.9 & 265 & --  \\
% \textbf{PruMerge} & Vicuna-7B & 336 & 64 & 67.4 & 51.9 & 50.1 & 68.1 & 54.0 & 65.3 & -- & 55.3 & 49.1 & 250 & --   \\
% \textbf{PruMerge++} & Vicuna-7B & 336 & 144 & -- & -- & -- & -- & -- & -- & -- & -- & -- & -- & -- \\
\textbf{LLaMA-VID}~\citep{li2024llama} & Vicuna-7B & 336 & 2 & -- & 55.5 & 54.2 & 68.8 & 49.0 & 83.1 & -- & -- & -- & -- & --  \\
\textbf{VoCo-LLaMA}~\citep{ye2024voco} & Vicuna-7B & 336 & 1 & -- & 55.6 & 54.6 & 68.4 & 31.7 & 80.8 & 1594 & 56.4 & 46.2 & 69 & -- \\
\textbf{TokenPacker}~\citep{li2024tokenpacker} & Vicuna-7B & 336 & 144 & 71.3 & 62.0 & 56.6 & 70.5 & 43.8 & 86.2 & 1716 & 63.9 & 53.4 & 303 & 59.9 {\color[HTML]{C2183A} \fontsize{5pt}{6pt}\selectfont (-1.5)} \\
\textbf{TokenPacker}~\citep{li2024tokenpacker} & Vicuna-7B & 336 & 36 & -- & 58.6 & 50.2 & -- & -- & 83.7 & -- & 62.8 & -- & -- & --   \\
\textbf{LLaVA-Mini}~\citep{zhang2025llava} & Vicuna-7B & 336 & 144 & 58.1 & 56.3 & 14.8 & 25.3 & 26.0 & 82.3 & 1325 & 24.8 & -- & 132 & -- \\
\textbf{LLaVA-Mini}~\citep{zhang2025llava} & Vicuna-7B & 336 & 64 & - & 56.6 & 10.4 & 27.4 & 28.1 & 82.3 & 1324 & 23.8 & -- & 145 & --   \\
\rowcolor{mygray}\multicolumn{15}{c}{\scriptsize\textit{\textbf{Ours}}} \\
% \rowcolor{blue!7}
\textbf{LLaVA-v1.5 + TCD} & Vicuna-7B & 336 & 256 & 72.7 & 61.4 & 54.1 & 69.8 & 57.0 & 85.8 & 1807 & 66.1 & 54.8 & 310 & 61.7 {\color[HTML]{18A6C2} \fontsize{5pt}{6pt}\selectfont (+0.3)} \\
% \rowcolor{blue!7}
\textbf{LLaVA-v1.5 + TCD} & Vicuna-7B & 336 & 192 & 71.6 & 60.9 & 54.0 & 70.0 & 56.9 & 85.3 & 1813 & 65.8 & 54.6 & 304 & 61.4 {\color[HTML]{18A6C2} \fontsize{5pt}{6pt}\selectfont (+0.0)} \\
% \rowcolor{blue!7}
\textbf{LLaVA-v1.5 + TCD} & Vicuna-7B & 336 & 128 & 69.7 & 59.9 & 54.9 & 70.8 & 56.6 & 84.5 & 1861 & 65.6 & 54.3 & 299 & 61.3 {\color[HTML]{C2183A} \fontsize{5pt}{6pt}\selectfont (-0.1)}    \\
% \rowcolor{blue!7}
\textbf{LLaVA-v1.5 + TCD} & Vicuna-7B & 336 & 64 & 66.1 & 57.1 & 55.1 & 71.1 & 54.8 & 79.2 & 1809 & 64.2 & 53.0 & 286 & 59.4 {\color[HTML]{C2183A} \fontsize{5pt}{6pt}\selectfont (-2.0)}  \\
% \rowcolor{blue!7}
\textbf{LLaVA-v1.5 + TCD} & Vicuna-7B & 336 & 36 & 62.1 & 54.9 & 55.2 & 71.3 & 53.6 & 75.8 & 1747 & 62.4 & 51.5 & 262 & 57.5 {\color[HTML]{C2183A} \fontsize{5pt}{6pt}\selectfont (-3.9)}  \\
% \multicolumn{3}{c|}{\textcolor{black!60}{$\Delta$ \textit{compare to LLaVA-v1.5-7B}}} & \textcolor{red}{0.17\%} & \textcolor{blue}{-0.9} & \textcolor{blue}{-1.1} & \textcolor{red}{+6.1} & \textcolor{red}{+3.6} & \textcolor{blue}{-1.3} & \textcolor{blue}{-1.5} & \textcolor{blue}{-44.7} & \textcolor{red}{+1.3} & \textcolor{blue}{-0.1} & \textcolor{red}{+5.5} & \textcolor{red}{+1.6} \\
[0.2mm]\specialrule{0.05pt}{0.01pt}{0.0pt} 
% \midrule
\textbf{LLaVA-v1.5 + LCD} & Vicuna-7B & 336 & 256 & 72.6 & 62.0 & 57.4 & 69.8 & 56.8 & 86.1 & 1834 & 64.3 & 56.0 & 312 & 62.2 {\color[HTML]{18A6C2} \fontsize{5pt}{6pt}\selectfont (+0.8)} \\ 
\textbf{LLaVA-v1.5 + LCD} & Vicuna-7B & 336 & 192 & 71.3 & 61.5 & 57.6 & 70.0 & 56.7 & 85.3 & 1830 & 64.4 & 55.9 & 316 & 62.0 {\color[HTML]{18A6C2} \fontsize{5pt}{6pt}\selectfont (+0.6)} \\ 
\textbf{LLaVA-v1.5 + LCD} & Vicuna-7B & 336 & 128 & 69.2 & 60.6 & 57.9 & 69.8 & 56.3 & 84.2 & 1832 & 64.1 & 55.3 & 306 & 61.3 {\color[HTML]{C2183A} \fontsize{5pt}{6pt}\selectfont (-0.1)} \\ 
\textbf{LLaVA-v1.5 + LCD} & Vicuna-7B & 336 & 64 & 66.0 & 58.3 & 57.8 & 69.7 & 54.3 & 81.2 & 1794 & 62.1 & 52.9 & 280 & 59.4 {\color[HTML]{C2183A} \fontsize{5pt}{6pt}\selectfont (-2.0)} \\ 
\textbf{LLaVA-v1.5 + LCD} & Vicuna-7B & 336 & 36 & 62.8 & 56.5 & 56.5 & 70.3 & 52.9 & 77.7 & 1711 & 60.7 & 51.0 & 265 & 57.6 {\color[HTML]{C2183A} \fontsize{5pt}{6pt}\selectfont (-3.8)}\\ 
% \multicolumn{3}{c|}{\textcolor{black!60}{$\Delta$ \textit{compare to LLaVA-v1.5-13B}}} & \textcolor{red}{11.1\%} & \textcolor{red}{+0.4} & \textcolor{blue}{-0.2} & \textcolor{red}{+8.5} & \textcolor{red}{+2.9} & \textcolor{red}{+0.9} & \textcolor{blue}{-0.6} & \textcolor{blue}{-33.9} & \textcolor{red}{+3.2} & \textcolor{red}{+1.6} & \textcolor{red}{+5.9} & \textcolor{red}{+2.4} \\[0.2mm]\specialrule{0.05pt}{0pt}{0pt}
% \begin{tabular}[c]{@{}l@{}}\textbf{LLaVA-Mini*}\\ \textbf{(Image \& Video)}\end{tabular} & \begin{tabular}[c]{@{}l@{}}LLaMA-3.1-\\ 8B-Instruct\end{tabular} & 336 & 1 & 79.0 & 61.3 & 57.4 & 83.1 & 58.5 & 85.3 & 1522.7 & 71.6 & 63.0 & 70.2 & 60.7 \\
% [0.2mm]\specialrule{0.05pt}{0.01pt}{0.0pt} 
% % \midrule
% \textbf{LLaVA-v1.5 + TCD} & Vicuna-13B &  \\ 
% \textbf{LLaVA-v1.5 + TCD} & Vicuna-13B &  \\ 
% \textbf{LLaVA-v1.5 + TCD} & Vicuna-13B &  \\ 
% \textbf{LLaVA-v1.5 + TCD} & Vicuna-13B &  \\ 
% \textbf{LLaVA-v1.5 + TCD} & Vicuna-13B &  \\ 
\bottomrule 
\end{tabular}
% \vspace{-5mm}
\end{table*}
Table~\ref{table:main_result} presents experimental results on 10 representative visual benchmarks.
%where both TCD and LCD use DART for training and inference. 
Notably, our method and MQT-LLaVA are among the few approaches enabling flexible control over token compression ($36 \sim 256$ tokens) with a single trained model, adapting efficiently to varying resource constraints.
When retaining $128$ tokens, our framework achieves performance comparable to vanilla LLaVA-v1.5-7B, while surpassing it with $192+$ visual tokens, strongly suggesting significant redundancy in visual tokens. Compared to other training-aware methods involving model modifications (\emph{e.g.}, MQT-LLaVA, TokenPacker), our approach maintains superior average performance, particularly excelling on MME, MMBench, and VQA V2. This indicates that effective training strategies are as crucial as architectural modifications for token compression.
The results also demonstrate our method's robustness across compression ratios: with just 64 tokens, performance degrades by merely $2\%$ versus vanilla LLaVA, while maintaining minimal ($<1\%$) variation at $128 \sim 256$ visual tokens.

\subsection{Efficiency}
\label{sec: efficiency}
% \begin{table}[!ht]
% \centering
% \caption{Inference efficiency analysis of \algname including Token number, KV cache memory, CUDA time, and the FLOPs. $\Delta$ denotes the reduction ratio.}
% \resizebox{0.95\linewidth}{!}{
% \begin{tabular}{l|c|c@{\hspace{0.5em}}c|c@{\hspace{0.5em}}c|cc}
% \toprule
% \textbf{Method} & \textbf{Token} & \textbf{KV cache (MB) $\downarrow$} & \textbf{$\Delta$} & \textbf{CUDA Time (s)} \textbf{$\downarrow$}  & \textbf{$\Delta$} & \textbf{FLOPs (T) $\downarrow$} & \textbf{$\Delta$} \\
% \midrule
% LLaVA-v1.5-7B & 576 & 367.2 & -- & 1103.5 & -- & 9.3 & --    \\
% % Full Caching & 576 & 302.4 & - & 154.9 & 64.8\% & 1.2 & 87.5\%     \\
% \midrule
% \algname + FastV & 64 & 40.9 & 88.9\% & 749.1 & 32.1\% & 1.5 & 83.9\% \\
% \algname + DART & 64 & 40.9 & 88.9\% & 744.3 & 32.6\% & 1.5 & 83.9\%  \\
% \algname + Radnom & 64 & 40.9 & 88.9\% & 697.3 & 36.8\% & 1.5 & 83.9\%   \\
% \bottomrule
% \end{tabular}}
% \label{tab:inference_efficiency}
% \end{table}
\begin{table*}[!ht]
\centering
% \vspace{-5mm}
\caption{Inference efficiency analysis of \texttt{\algname}. 
$\Delta$ denotes the reduction ratio. 
All experiments are on POPE ($8,910$ samples) using an A100 GPU. Token compression is fixed at the 2nd layer.}% of the language model.}
% \vspace{1mm}
% \setlength{\tabcolsep}{2pt}
\resizebox{0.81\linewidth}{!}{
\begin{tabular}{l|c|c@{\hspace{0.5em}}c|c@{\hspace{0.5em}}c|cc}
\toprule
\rowcolor{gray!7}
\textbf{Method} & \textbf{Visual Tokens} & \multicolumn{2}{c|}{\textbf{KV cache (MB) $\downarrow$}} & \multicolumn{2}{c|}{\textbf{CUDA Time (s) $\downarrow$}} & \multicolumn{2}{c}{\textbf{FLOPs (T) $\downarrow$}} \\
\cmidrule(lr){3-4} \cmidrule(lr){5-6} \cmidrule(lr){7-8}
 & & \textbf{Value} & \textbf{$\Delta$} & \textbf{Value} & \textbf{$\Delta$} & \textbf{Value} & \textbf{$\Delta$} \\
\midrule
LLaVA-v1.5-7B & 576 & 367.2 & -- & 1103.5 & -- & 9.3 & -- \\
% \addlinespace[0.2em]
\rowcolor{green!5}
\texttt{\algname} + FastV~\citep{chen2024image} & 64 & 40.9 & 88.9\% & 749.1 & 32.1\% & 1.5 & 83.9\% \\
\rowcolor{orange!5}
\texttt{\algname} + DART~\citep{wen2025stop} & 64 & 40.9 & 88.9\% & 744.3 & 32.6\% & 1.5 & 83.9\% \\
\rowcolor{blue!5}
\texttt{\algname} + Random & 64 & 40.9 & 88.9\% & 697.3 & 36.8\% & 1.5 & 83.9\% \\
\bottomrule
\end{tabular}}
\label{tab:inference_efficiency}
% \vspace{-1mm}
\end{table*}
% All experiments are conducted on POPE Benchmark with 8,910 samples and token compression layer is fixed 2nd layer in language model.
We discuss the efficiency of \texttt{\algname}, including \textbf{training-time efficiency} and \textbf{inference-time efficiency}.
As shown in Table~\ref{tab:training_efficiency} in Appendix~\ref{app_sec:training_details}, in contrast to other training-aware token compression method, especially those that alter the model architecture, our approach only requires supervised fine-tuning, completing training in approximately $12$ hours on $8$ A100 GPUs.
Most token compression methods that modify the model architecture require two or even three training stages. Furthermore, the replacement or addition of new model components necessitates more training iterations to properly adapt these parameters, resulting in significantly greater computational expenditure (\emph{e.g.}, $30 \sim 48$ hours on $8$ A100 GPUs). This substantially increases the overall training cost.
Focus on the inference-time efficiency, which in our framework is primarily influenced by the token compression strategies employed during the inference process.
Table~\ref{tab:inference_efficiency} presents the KV cache memory usage, CUDA time, and FLOPs obtained when applying three different token compression methods during inference for models trained using Token Consistency Distillation. The FLOPs are calculated using \texttt{calflops}~\citep{calflops}, while the KV cache memory usage is estimated with the help of \texttt{LLM-Viewer}~\citep{yuan2024llm}.
It can be observed that when retaining 64 tokens, all methods achieve improvements in reducing KV cache memory, FLOPs, and latency. In particular, Random token compression, which incurs no additional computational overhead, achieves an actual speedup of nearly $1.6 \times$.

\section{Analyses}
\label{sec:analyses}
% 1. 消融实验：
% （1）去掉distillation
%  (2) 
% 2. 方法间的泛化性：
% （1）在用dart训练的不仅在用dart推理（主表），也在fastv 和 random上测试，看提升
% （2）用random和fastv也各自训练一个模型，然后不仅各自用自己的方法推理，也互相测试，看提升多少
In this section, we conduct a thorough investigation aimed at addressing the following key questions:
\textbf{(1)} How much does the weight-sharing teacher guidance contribute to performance?
\textbf{(2)} What occurs without progressive learning in token/layer-wise dimensions?
\textbf{(3)} How generalizable is \texttt{\algname} across various token compression strategies?
\textbf{(4)} Is extreme token compression ($1$ or $2$ tokens) necessary?

\begin{table*}[!t]
% \vspace{-2mm}
\caption{
Ablation study on Token Consistency Distillation. ``w/o Distillation Loss'' disables teacher supervision by zeroing $\mathcal{L}_{\text{TCD}}$. 
``w/o Progressive Compression Ratio'' uses a fixed $88.9\%$ compression ratio. DART is employed as the token compression strategy during training.
}
% \vspace{-1mm}
\setlength{\tabcolsep}{2pt}
\renewcommand{\arraystretch}{1.1} 
\resizebox{\linewidth}{!}{  
\tiny
\centering
\begin{tabular}{l  | cccc ccccC{0.92cm}C{0.96cm} | c}
\toprule
\textbf{Method} & \textbf{$\text{VQA}^\text{V2}$} & \textbf{GQA} & \textbf{VizWiz} & \textbf{$\text{SQA}^\text{I}$} & \textbf{$\text{VQA}^\text{T}$} & \textbf{POPE} & \textbf{$\text{MME}$} &  \textbf{MMB}  & \textbf{MMB-CN} & \textbf{OCRBench} & \textbf{Avg. (\%)}  \\
% \midrule
% \midrule
% QT-LLaVA (Baseline) & 51.1&  \textbf{68.1}&  \textbf{76.8}$^*$& 61.5$^*$&  84.1&  1431.2&  348.2&  34.3&  64.0& 63.9& 27.9  & 58.8 \\
 % \rowcolor{black!10} 
 \hline
 \rowcolor{mygray}
\multicolumn{12}{c}{\textit{Retain 128 Tokens} \ $\fg{(\downarrow 77.8\%)}$} \\
 \textbf{TCD (Ours)} & 69.7 & 59.9 & 54.9 & 70.8 & 56.6 & 84.5 & 1861 & 65.6 & 54.3 & 299 & 61.3 \\
 w/o Distillation Loss & 67.2 & 60.4 & 53.2 & 68.6 & 57.3 & 83.2 & 1745 & 63.8 & 53.5 & 289 & 59.8 {\color[HTML]{C2183A} \fontsize{5pt}{6pt}\selectfont (-1.5)} \\
 w/o Progressive Compression Ratio & 67.1 & 58.9 & 49.7 & 70.0 & 54.1 & 84.3 & 1788 & 63.8 & 51.5 & 277 & 59.1 {\color[HTML]{C2183A} \fontsize{5pt}{6pt}\selectfont (-2.2)} \\
 \hline
 \rowcolor{mygray}
\multicolumn{12}{c}{\textit{Retain 64 Tokens} \ $\fg{(\downarrow 88.9\%)}$} \\
 \textbf{TCD (Ours)} & 66.1 & 57.1 & 55.1 & 71.1 & 54.8 & 79.2 & 1809 & 64.2 & 53.0 & 286 & 59.4 \\
 w/o Distillation Loss & 65.7 & 57.9 & 53.6 & 69.8 & 54.9 & 78.3 & 1671 & 61.6 & 52.1 & 272 & 58.1 {\color[HTML]{C2183A} \fontsize{5pt}{6pt}\selectfont (-1.3)} \\
 w/o Progressive Compression Ratio & 64.3 & 57.9 & 50.3 & 70.3 & 52.0 & 82.7 & 1774 & 62.8 & 52.5 & 255 & 58.2 {\color[HTML]{C2183A} \fontsize{5pt}{6pt}\selectfont (-1.2)} \\
\bottomrule
\end{tabular}}
\label{tab:abalation_study1}
% \vspace{-4mm}
\end{table*}
% ``w/o Distillation Loss'' refers to the configuration where the distillation loss is zeroed out, effectively disabling teacher model supervision. ``w/o Progressive Compression Ratio'' denotes using a constant 88.9\% compression ratio.

\subsection{Ablation Studies}
\label{sec:ablation}
\begin{table*}[!ht]
% \vspace{-2mm}
\caption{
Ablation study on Layer Consistency Distillation. 
``w/o Progressive Compression Layer'': fixed at the 2nd layer. DART is employed as the token compression strategy during training.
}
% \vspace{-1mm}
\centering
\setlength{\tabcolsep}{2pt}
\renewcommand{\arraystretch}{1.1} 
\resizebox{\linewidth}{!}{  
\tiny
\centering
\begin{tabular}{l  | cccc ccccC{0.92cm}C{0.96cm} | c}
\toprule
\textbf{Method} & \textbf{$\text{VQA}^\text{V2}$} & \textbf{GQA} & \textbf{VizWiz} & \textbf{$\text{SQA}^\text{I}$} & \textbf{$\text{VQA}^\text{T}$} & \textbf{POPE} & \textbf{$\text{MME}$} &  \textbf{MMB}  & \textbf{MMB-CN} & \textbf{OCRBench} & \textbf{Avg. (\%)}  \\
% \midrule
% \midrule
% QT-LLaVA (Baseline) & 51.1&  \textbf{68.1}&  \textbf{76.8}$^*$& 61.5$^*$&  84.1&  1431.2&  348.2&  34.3&  64.0& 63.9& 27.9  & 58.8 \\
 % \rowcolor{black!10} 
\hline
\rowcolor{mygray}
\multicolumn{12}{c}{\textit{Retain 128 Tokens} \ $\fg{(\downarrow 77.8\%)}$} \\
 \textbf{LCD (Ours)} & 69.2 & 60.6 & 57.9 & 69.8 & 56.3 & 84.2 & 1832 & 64.1 & 55.3 & 306 & 61.3 \\
 w/o Distillation Loss & 67.1 & 60.6 & 55.4 & 70.3 & 56.1 & 84.3 & 1761 & 62.9 & 55.4 & 301 & 60.5 {\color[HTML]{C2183A} \fontsize{5pt}{6pt}\selectfont (-0.8)} \\
 w/o Progressive Compression Layer & 68.7 & 59.3 & 54.3 & 70.6 & 56.2 & 82.2 & 1776 & 63.1 & 54.9 & 298 & 60.3 {\color[HTML]{C2183A} \fontsize{5pt}{6pt}\selectfont (-1.0)} \\
 \hline
 \rowcolor{mygray}
\multicolumn{12}{c}{\textit{Retain 64 Tokens} \ $\fg{(\downarrow 88.9\%)}$} \\
 \textbf{LCD (Ours)} & 66.0 & 58.3 & 57.8 & 69.7 & 54.3 & 81.2 & 1794 & 62.1 & 52.9 & 280 & 59.4 \\
 w/o Distillation Loss & 64.4 & 58.2 & 55.9 & 69.6 & 54.8 & 80.9 & 1735 & 61.3 & 52.7 & 275 & 58.7 {\color[HTML]{C2183A} \fontsize{5pt}{6pt}\selectfont (-0.7)} \\
 w/o Progressive Compression Layer & 63.5 & 56.5 & 54.7 & 71.5 & 54.2 & 75.6 & 1734 & 61.9 & 51.7 & 260 & 57.8 {\color[HTML]{C2183A} \fontsize{5pt}{6pt}\selectfont (-1.6)} \\
\bottomrule
\end{tabular}}
\label{tab:abalation_study2}
% \vspace{-4mm}
\end{table*}
To validate the guiding role of the weight-sharing teacher model (RQ1) and the importance of the progressive learning strategy in training-aware token compression (RQ2), we conducted ablation studies on both components.
Specifically, for token consistency distillation (TCD) and layer consistency distillation (LCD), in addition to eliminating the distillation loss to remove the teacher model’s influence, we also fix the compression ratio (\emph{e.g.}, $88.9\%$) and compression layer (\emph{e.g.}, the 2nd layer) to eliminate the progressive learning strategy in both the token-wise and layer-wise dimensions. This essentially degenerates the process into imposing a significant perturbation in the feature space, and the model is trained to adapt in the parameter space, like direct training.
As demonstrated in Tables~\ref{tab:abalation_study1} and \ref{tab:abalation_study2}, the experimental results indicate that without teacher guidance, both TCD and LCD exhibit significant performance degradation across multiple benchmarks on average, with particularly pronounced declines on vision-centric benchmarks such as MME and MMBench.
% add some common distillation ablation exp...

\subsection{How Well Does Proposed Framework Generalize across Different Methods?}
\label{sec:generalize}
\begin{figure*}[!th]
    % \vspace{-1mm}
    \centering
    \includegraphics[width=1.0\linewidth]{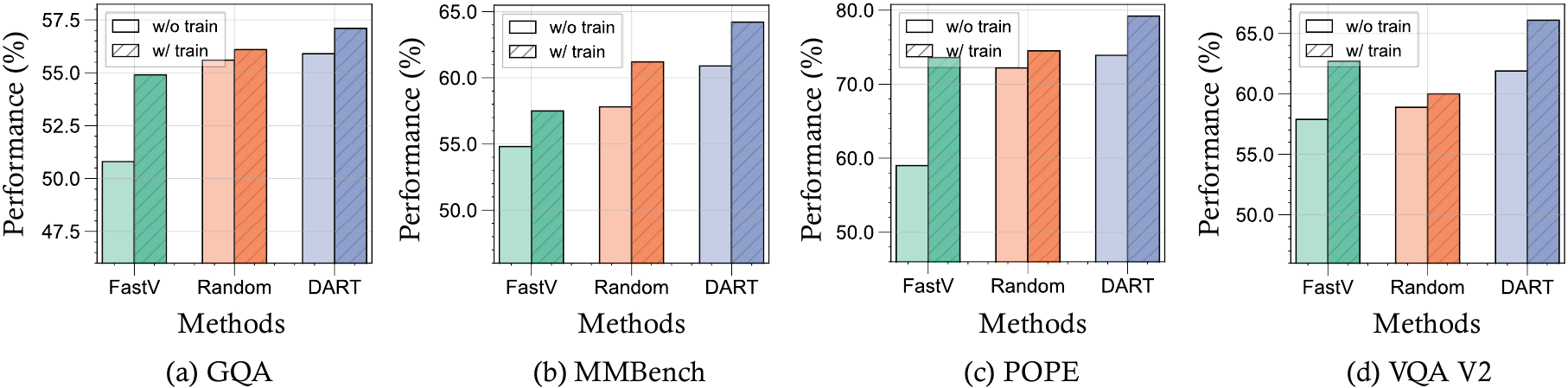}
    % \vspace{-7mm}
    \caption{Following LLaVA-v1.5's architecture and data, we apply \textbf{DART} for token consistency distillation. ``w/o train'' denotes vanilla LLaVA. At inference, all methods use $88.9\%$ token compression.}
    % \vspace{-1mm}
    \label{fig:generalized_exp}
\end{figure*}
% Following model architecture and training data of LLaVA-v1.5-7B, we apply the \textbf{DART} for token consistency distillation during training. Here, ``w/o train'' denotes the vanilla LLaVA. During inference, all methods achieve a token compression ratio of $88.9\%$.
Beyond the generalization across different compression ratios observed in our comparative experiments of Sec.~\ref{sec:main_results}, we further investigate the adaptability of our proposed framework to diverse token compression strategies (RQ3).
Specifically, as outlined in Sec.~\ref{sec:implementation}, we integrate three plug-and-play token compression strategies into our framework during training. We then conduct cross-strategy evaluations to assess how models trained with one specific strategy (\emph{e.g.}, FastV, DART) perform when applied with alternative compression methods during inference.
%, and (ii) whether such cross-application yields performance gains compared to directly applying those compression methods to vanilla models.
% As shown in Figure~\ref{fig:generalized_exp}, across all benchmarks, the same token compression method consistently improves model performance when trained with token consistency distillation.
% Remarkably, even when trained only on DART-based token compression, the model generalizes to other compression approaches (FastV and Random), achieving measurable performance improvements in all benchmarks.
As shown in Figure~\ref{fig:generalized_exp}, token consistency distillation consistently improves model performance across all benchmarks and compression methods. Notably, even when trained solely with DART-based compression, the model generalizes well to FastV and Random compression, yielding consistent performance gains.
Furthermore, after training with our proposed framework, the performance gap between token compression strategies is significantly reduced. Notably, previously underperforming strategies exhibit more substantial improvements than their stronger counterparts.
For additional experiments and discussions on the generalization capability of our method, please refer to Appendix~\ref{app_sec:more_generalizability_experiments}.

% \vspace{-2mm}
\subsection{Is Extreme Token Compression Necessary?}
\begin{figure*}[!h]
    % \vspace{-1mm}
    \centering
    \includegraphics[width=1.0\linewidth]{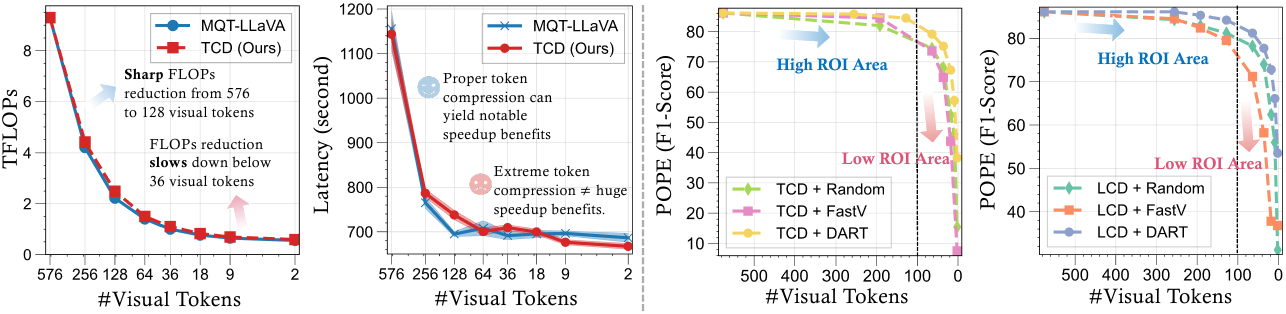}
    % \vspace{-7mm}
    \caption{All experiments use the model trained following LLaVA-v1.5. FLOPs and latency are measured on the POPE. Visual token and latency experiments are repeated three times for reliability.}
    % \vspace{-1mm}
    \label{fig:extreme_compression_analyses}
\end{figure*}
% All experiments are conducted based on the model trained following the LLaVA-v1.5-7B framework. Measurements of FLOPs and latency are carried out using the POPE benchmark. The experiments related to visual tokens and latency are repeated three times to ensure reliability.

We observe that many token compression methods pursue aggressive compression (\emph{e.g.}, one or two tokens).
Table~\ref{table:main_result} shows that extreme token compression still leads to notable performance degradation for many methods.
While this consistently reduces KV cache memory, it raises the question: \emph{does such extreme compression always translate to faster inference} (RQ4)?
To answer this, we conducted detailed analysis and experiments.
Figure~\ref{fig:extreme_compression_analyses} shows the relationship between FLOPs and the number of retained visual tokens.
When reducing tokens from the full set ($576$ tokens) to $128$, FLOPs drop significantly—from $9.3$T to around $2$T.
However, under more extreme compression, FLOPs reduction becomes noticeably smaller.
A similar or even more pronounced trend is observed between token count and actual latency.
In some cases, retaining $64$ tokens yields better performance than fewer tokens (\emph{e.g.}, $36$ or $18$).
A possible hypothesis is that overly fragmented feature slices increase memory access time.
Moreover, reducing tokens to $64$ largely preserves vanilla model performance.
We refer to this range as the High Return-on-Investment \textbf{(High ROI)} Area.
Further reduction beyond $64$ offers only marginal latency gains but sharply degrades performance—this is the \textbf{Low ROI} Area.
Model efficiency depends on whether it is computation or memory-bound.
With heavily reduced tokens, GPU compute is underutilized and latency is dominated by memory access, making the system memory-bound, where further reduction brings little speedup.
Overall, we argue that extreme compression is unnecessary; instead, focus should be on balancing latency and performance.

\vspace{-1mm}
\section{Conclusion}
\vspace{-1mm}
% In this paper, we propose \algname, a learning framework that advances efficient multi-modal large language model (MLLMs) training through Progressive Consistency Distillation, which can seamlessly integrate with existing token compression strategies without requiring any modifications to the model architecture.
% ---------- v2 ------------
% In this paper, we propose \algname, a learning framework that advances efficient multi-modal large language model (MLLMs) training through Progressive Consistency Distillation, which can seamlessly integrate with existing token compression strategies without requiring any modifications to the model architecture, achieving efficiency in both training and inference.
% Experimental results demonstrate that models trained with our framework can achieve comparable average performance to the vanilla model using only 128 visual tokens, and notably, on 4 out of 10 benchmarks, they even surpass the vanilla model despite using significantly fewer visual tokens.
% Moreover, extensive experiments confirm the effectiveness of \algname, while highlighting its robustness under varying compression ratios and its generalizability across diverse compression strategies.
% Finally, we analyze and demonstrate that excessive token compression may lead to a poor latency-performance trade-off.
% --------- v3 ------------
In this paper, we propose \texttt{\algname}, a learning framework that enhances the efficiency of multi-modal large language models (MLLMs) via Progressive Consistency Distillation. It integrates with existing token compression strategies without modifying the model architecture, achieving efficiency in both training and inference. Experimental results demonstrate that MLLMs trained with our framework achieve comparable average performance to the vanilla model using only $128$ visual tokens. Notably, on 4 out of 10 visual understanding benchmarks, our approach even outperforms the vanilla model despite using significantly fewer visual tokens. Furthermore, extensive experiments validate the effectiveness of \texttt{\algname}, highlighting its robustness across token compression ratios and generalization across strategies. Meanwhile, our analysis reveals that while token compression offers significant benefits, excessive compression may lead to a poor latency-performance trade-off, underscoring the importance of balancing compression levels for optimal performance.

% \section*{Acknowledgements}
% This research was supported by the Shanghai Science and Technology Program (Grant No. 25ZR1402278) and Shanghai Artificial Intelligence Laboratory. Besides, we thank Huawei Ascend Cloud Ecological Development Project for the support of Ascend 910 processors.

% \clearpage
\bibliographystyle{plain}
\bibliography{neurips_2025}

\begin{thebibliography}{10}

\bibitem{alayrac2022flamingo}
Jean-Baptiste Alayrac, Jeff Donahue, Pauline Luc, Antoine Miech, Iain Barr, Yana Hasson, Karel Lenc, Arthur Mensch, Katherine Millican, Malcolm Reynolds, et~al.
\newblock Flamingo: a visual language model for few-shot learning.
\newblock {\em Advances in neural information processing systems}, 35:23716--23736, 2022.

\bibitem{bai2023qwen}
Jinze Bai, Shuai Bai, Yunfei Chu, Zeyu Cui, Kai Dang, Xiaodong Deng, Yang Fan, Wenbin Ge, Yu~Han, Fei Huang, et~al.
\newblock Qwen technical report.
\newblock {\em arXiv preprint arXiv:2309.16609}, 2023.

\bibitem{Bai:Qwen-VL}
Jinze Bai, Shuai Bai, Shusheng Yang, Shijie Wang, Sinan Tan, Peng Wang, Junyang Lin, Chang Zhou, and Jingren Zhou.
\newblock {Qwen-VL}: {A} frontier large vision-language model with versatile abilities.
\newblock {\em arXiv preprint arXiv:2308.12966}, 2023.

\bibitem{bai2025qwen2}
Shuai Bai, Keqin Chen, Xuejing Liu, Jialin Wang, Wenbin Ge, Sibo Song, Kai Dang, Peng Wang, Shijie Wang, Jun Tang, et~al.
\newblock Qwen2. 5-vl technical report.
\newblock {\em arXiv preprint arXiv:2502.13923}, 2025.

\bibitem{cha2024honeybee}
Junbum Cha, Wooyoung Kang, Jonghwan Mun, and Byungseok Roh.
\newblock Honeybee: Locality-enhanced projector for multimodal llm.
\newblock In {\em Proceedings of the IEEE/CVF Conference on Computer Vision and Pattern Recognition}, pages 13817--13827, 2024.

\bibitem{chen2025variation}
Junjie Chen, Xuyang Liu, Zichen Wen, Yiyu Wang, Siteng Huang, and Honggang Chen.
\newblock Variation-aware vision token dropping for faster large vision-language models.
\newblock {\em arXiv preprint arXiv:2509.01552}, 2025.

\bibitem{chen2024image}
Liang Chen, Haozhe Zhao, Tianyu Liu, Shuai Bai, Junyang Lin, Chang Zhou, and Baobao Chang.
\newblock An image is worth 1/2 tokens after layer 2: Plug-and-play inference acceleration for large vision-language models.
\newblock In {\em European Conference on Computer Vision}, pages 19--35. Springer, 2024.

\bibitem{chen2024mj}
Zhaorun Chen, Yichao Du, Zichen Wen, Yiyang Zhou, Chenhang Cui, Zhenzhen Weng, Haoqin Tu, Chaoqi Wang, Zhengwei Tong, Qinglan Huang, et~al.
\newblock Mj-bench: Is your multimodal reward model really a good judge for text-to-image generation?
\newblock {\em arXiv preprint arXiv:2407.04842}, 2024.

\bibitem{chensafewatch}
Zhaorun Chen, Francesco Pinto, Minzhou Pan, and Bo~Li.
\newblock Safewatch: An efficient safety-policy following video guardrail model with transparent explanations.
\newblock In {\em The Thirteenth International Conference on Learning Representations}.

\bibitem{chen2024expanding}
Zhe Chen, Weiyun Wang, Yue Cao, Yangzhou Liu, Zhangwei Gao, Erfei Cui, Jinguo Zhu, Shenglong Ye, Hao Tian, Zhaoyang Liu, et~al.
\newblock Expanding performance boundaries of open-source multimodal models with model, data, and test-time scaling.
\newblock {\em arXiv preprint arXiv:2412.05271}, 2024.

\bibitem{chen2024far}
Zhe Chen, Weiyun Wang, Hao Tian, Shenglong Ye, Zhangwei Gao, Erfei Cui, Wenwen Tong, Kongzhi Hu, Jiapeng Luo, Zheng Ma, et~al.
\newblock How far are we to gpt-4v? closing the gap to commercial multimodal models with open-source suites.
\newblock {\em Science China Information Sciences}, 67(12):220101, 2024.

\bibitem{Vicuna}
Wei-Lin Chiang, Zhuohan Li, Zi~Lin, Ying Sheng, Zhanghao Wu, Hao Zhang, Lianmin Zheng, Siyuan Zhuang, Yonghao Zhuang, Joseph~E. Gonzalez, Ion Stoica, and Eric~P. Xing.
\newblock Vicuna: An opensource chatbot impressing gpt-4 with 90\% chatgpt quality.
\newblock {\em ArXiv preprint}, 2023.

\bibitem{chu2023mobilevlm}
Xiangxiang Chu, Limeng Qiao, Xinyang Lin, Shuang Xu, Yang Yang, Yiming Hu, Fei Wei, Xinyu Zhang, Bo~Zhang, Xiaolin Wei, et~al.
\newblock Mobilevlm: A fast, reproducible and strong vision language assistant for mobile devices.
\newblock {\em arXiv preprint arXiv:2312.16886}, 1(2):3, 2023.

\bibitem{chu2024mobilevlm}
Xiangxiang Chu, Limeng Qiao, Xinyu Zhang, Shuang Xu, Fei Wei, Yang Yang, Xiaofei Sun, Yiming Hu, Xinyang Lin, Bo~Zhang, et~al.
\newblock Mobilevlm v2: Faster and stronger baseline for vision language model.
\newblock {\em arXiv preprint arXiv:2402.03766}, 2024.

\bibitem{instruct-blip}
Wenliang Dai, Junnan Li, Dongxu Li, Anthony Meng~Huat Tiong, Junqi Zhao, Weisheng Wang, Boyang Li, Pascale Fung, and Steven Hoi.
\newblock Instructblip: Towards general-purpose vision-language models with instruction tuning, 2023.

\bibitem{instructblip}
Wenliang Dai, Junnan Li, Dongxu Li, Anthony Meng~Huat Tiong, Junqi Zhao, Weisheng Wang, Boyang Li, Pascale Fung, and Steven Hoi.
\newblock Instructblip: Towards general-purpose vision-language models with instruction tuning, 2023.

\bibitem{dao2023flashattention}
Tri Dao.
\newblock Flashattention-2: Faster attention with better parallelism and work partitioning.
\newblock {\em arXiv preprint arXiv:2307.08691}, 2023.

\bibitem{dao2022flashattention}
Tri Dao, Dan Fu, Stefano Ermon, Atri Rudra, and Christopher R{\'e}.
\newblock Flashattention: Fast and memory-efficient exact attention with io-awareness.
\newblock {\em Advances in neural information processing systems}, 35:16344--16359, 2022.

\bibitem{dong2024internlm}
Xiaoyi Dong, Pan Zhang, Yuhang Zang, Yuhang Cao, Bin Wang, Linke Ouyang, Songyang Zhang, Haodong Duan, Wenwei Zhang, Yining Li, et~al.
\newblock Internlm-xcomposer2-4khd: A pioneering large vision-language model handling resolutions from 336 pixels to 4k hd.
\newblock {\em Advances in Neural Information Processing Systems}, 37:42566--42592, 2024.

\bibitem{fani2011implications}
Tayebeh Fani and Farid Ghaemi.
\newblock Implications of vygotsky's zone of proximal development (zpd) in teacher education: Zptd and self-scaffolding.
\newblock {\em Procedia-Social and Behavioral Sciences}, 29:1549--1554, 2011.

\bibitem{fu2023mme}
Chaoyou Fu, Peixian Chen, Yunhang Shen, Yulei Qin, Mengdan Zhang, Xu~Lin, Jinrui Yang, Xiawu Zheng, Ke~Li, Xing Sun, et~al.
\newblock Mme: A comprehensive evaluation benchmark for multimodal large language models.
\newblock {\em arXiv preprint arXiv:2306.13394}, 2023.

\bibitem{balanced_vqa_v2}
Yash Goyal, Tejas Khot, Douglas Summers{-}Stay, Dhruv Batra, and Devi Parikh.
\newblock Making the {V} in {VQA} matter: Elevating the role of image understanding in {V}isual {Q}uestion {A}nswering.
\newblock In {\em Conference on Computer Vision and Pattern Recognition (CVPR)}, 2017.

\bibitem{grattafiori2024llama}
Aaron Grattafiori, Abhimanyu Dubey, Abhinav Jauhri, Abhinav Pandey, Abhishek Kadian, Ahmad Al-Dahle, Aiesha Letman, Akhil Mathur, Alan Schelten, Alex Vaughan, et~al.
\newblock The llama 3 herd of models.
\newblock {\em arXiv preprint arXiv:2407.21783}, 2024.

\bibitem{he2024zipvl}
Yefei He, Feng Chen, Jing Liu, Wenqi Shao, Hong Zhou, Kaipeng Zhang, and Bohan Zhuang.
\newblock Zipvl: Efficient large vision-language models with dynamic token sparsification and kv cache compression.
\newblock {\em arXiv preprint arXiv:2410.08584}, 2024.

\bibitem{hinton2015distilling}
Geoffrey Hinton, Oriol Vinyals, and Jeff Dean.
\newblock Distilling the knowledge in a neural network.
\newblock {\em arXiv preprint arXiv:1503.02531}, 2015.

\bibitem{hu2024matryoshka}
Wenbo Hu, Zi-Yi Dou, Liunian Li, Amita Kamath, Nanyun Peng, and Kai-Wei Chang.
\newblock Matryoshka query transformer for large vision-language models.
\newblock {\em Advances in Neural Information Processing Systems}, 37:50168--50188, 2024.

\bibitem{hudson2019gqa}
Drew~A Hudson and Christopher~D Manning.
\newblock Gqa: A new dataset for real-world visual reasoning and compositional question answering.
\newblock In {\em CVPR}, 2019.

\bibitem{jiang2025kind}
Yutao Jiang, Qiong Wu, Wenhao Lin, Wei Yu, and Yiyi Zhou.
\newblock What kind of visual tokens do we need? training-free visual token pruning for multi-modal large language models from the perspective of graph.
\newblock {\em arXiv preprint arXiv:2501.02268}, 2025.

\bibitem{kullback1951information}
Solomon Kullback and Richard~A Leibler.
\newblock On information and sufficiency.
\newblock {\em The annals of mathematical statistics}, 22(1):79--86, 1951.

\bibitem{laurenccon2023obelics}
Hugo Lauren{\c{c}}on, Lucile Saulnier, L{\'e}o Tronchon, Stas Bekman, Amanpreet Singh, Anton Lozhkov, Thomas Wang, Siddharth Karamcheti, Alexander Rush, Douwe Kiela, et~al.
\newblock Obelics: An open web-scale filtered dataset of interleaved image-text documents.
\newblock {\em Advances in Neural Information Processing Systems}, 36:71683--71702, 2023.

\bibitem{li2024llava}
Bo~Li, Yuanhan Zhang, Dong Guo, Renrui Zhang, Feng Li, Hao Zhang, Kaichen Zhang, Peiyuan Zhang, Yanwei Li, Ziwei Liu, et~al.
\newblock Llava-onevision: Easy visual task transfer.
\newblock {\em arXiv preprint arXiv:2408.03326}, 2024.

\bibitem{li2023blip}
Junnan Li, Dongxu Li, Silvio Savarese, and Steven Hoi.
\newblock Blip-2: Bootstrapping language-image pre-training with frozen image encoders and large language models.
\newblock In {\em International conference on machine learning}, pages 19730--19742. PMLR, 2023.

\bibitem{li2024tokenpacker}
Wentong Li, Yuqian Yuan, Jian Liu, Dongqi Tang, Song Wang, Jie Qin, Jianke Zhu, and Lei Zhang.
\newblock Tokenpacker: Efficient visual projector for multimodal llm.
\newblock {\em arXiv preprint arXiv:2407.02392}, 2024.

\bibitem{li2024llama}
Yanwei Li, Chengyao Wang, and Jiaya Jia.
\newblock Llama-vid: An image is worth 2 tokens in large language models.
\newblock In {\em European Conference on Computer Vision}, pages 323--340. Springer, 2024.

\bibitem{li2024mini}
Yanwei Li, Yuechen Zhang, Chengyao Wang, Zhisheng Zhong, Yixin Chen, Ruihang Chu, Shaoteng Liu, and Jiaya Jia.
\newblock Mini-gemini: Mining the potential of multi-modality vision language models.
\newblock {\em arXiv:2403.18814}, 2024.

\bibitem{li-etal-2023-evaluating}
Yifan Li, Yifan Du, Kun Zhou, Jinpeng Wang, Xin Zhao, and Ji-Rong Wen.
\newblock Evaluating object hallucination in large vision-language models.
\newblock In {\em Proc. of EMNLP}, 2023.

\bibitem{lin2023video}
Bin Lin, Bin Zhu, Yang Ye, Munan Ning, Peng Jin, and Li~Yuan.
\newblock Video-llava: Learning united visual representation by alignment before projection.
\newblock {\em arXiv preprint arXiv:2311.10122}, 2023.

\bibitem{lin2023sphinx}
Ziyi Lin, Chris Liu, Renrui Zhang, Peng Gao, Longtian Qiu, Han Xiao, Han Qiu, Chen Lin, Wenqi Shao, Keqin Chen, et~al.
\newblock Sphinx: The joint mixing of weights, tasks, and visual embeddings for multi-modal large language models.
\newblock {\em arXiv preprint arXiv:2311.07575}, 2023.

\bibitem{liu2024improved}
Haotian Liu, Chunyuan Li, Yuheng Li, and Yong~Jae Lee.
\newblock Improved baselines with visual instruction tuning.
\newblock In {\em Proceedings of the IEEE/CVF Conference on Computer Vision and Pattern Recognition}, pages 26296--26306, 2024.

\bibitem{liu2024llavanext}
Haotian Liu, Chunyuan Li, Yuheng Li, Bo~Li, Yuanhan Zhang, Sheng Shen, and Yong~Jae Lee.
\newblock Llava-next: Improved reasoning, ocr, and world knowledge, January 2024.

\bibitem{liu2023llava}
Haotian Liu, Chunyuan Li, Qingyang Wu, and Yong~Jae Lee.
\newblock Visual instruction tuning.
\newblock In A.~Oh, T.~Neumann, A.~Globerson, K.~Saenko, M.~Hardt, and S.~Levine, editors, {\em Advances in Neural Information Processing Systems}, volume~36, pages 34892--34916. Curran Associates, Inc., 2023.

\bibitem{liu2024visual}
Haotian Liu, Chunyuan Li, Qingyang Wu, and Yong~Jae Lee.
\newblock Visual instruction tuning.
\newblock {\em Advances in neural information processing systems}, 2024.

\bibitem{liu2025compression}
Xuyang Liu, Ziming Wang, Yuhang Han, Yingyao Wang, Jiale Yuan, Jun Song, Bo~Zheng, Linfeng Zhang, Siteng Huang, and Honggang Chen.
\newblock Compression with global guidance: Towards training-free high-resolution mllms acceleration.
\newblock {\em arXiv preprint arXiv:2501.05179}, 2025.

\bibitem{liu2025shifting}
Xuyang Liu, Zichen Wen, Shaobo Wang, Junjie Chen, Zhishan Tao, Yubo Wang, Xiangqi Jin, Chang Zou, Yiyu Wang, Chenfei Liao, et~al.
\newblock Shifting ai efficiency from model-centric to data-centric compression.
\newblock {\em arXiv preprint arXiv:2505.19147}, 2025.

\bibitem{liu2023mmbench}
Yuan Liu, Haodong Duan, Yuanhan Zhang, Bo~Li, Songyang Zhang, Wangbo Zhao, Yike Yuan, Jiaqi Wang, Conghui He, Ziwei Liu, et~al.
\newblock Mmbench: Is your multi-modal model an all-around player?
\newblock {\em arXiv preprint arXiv:2307.06281}, 2023.

\bibitem{liu2024ocrbench}
Yuliang Liu, Zhang Li, Mingxin Huang, Biao Yang, Wenwen Yu, Chunyuan Li, Xu-Cheng Yin, Cheng-Lin Liu, Lianwen Jin, and Xiang Bai.
\newblock Ocrbench: on the hidden mystery of ocr in large multimodal models.
\newblock {\em Science China Information Sciences}, 67(12):220102, 2024.

\bibitem{lu2022learn}
Pan Lu, Swaroop Mishra, Tanglin Xia, Liang Qiu, Kai-Wei Chang, Song-Chun Zhu, Oyvind Tafjord, Peter Clark, and Ashwin Kalyan.
\newblock Learn to explain: Multimodal reasoning via thought chains for science question answering.
\newblock {\em NeurIPS}, 2022.

\bibitem{marr2010vision}
David Marr.
\newblock {\em Vision: A computational investigation into the human representation and processing of visual information}.
\newblock MIT press, 2010.

\bibitem{radford2021learning}
Alec Radford, Jong~Wook Kim, Chris Hallacy, Aditya Ramesh, Gabriel Goh, Sandhini Agarwal, Girish Sastry, Amanda Askell, Pamela Mishkin, Jack Clark, et~al.
\newblock Learning transferable visual models from natural language supervision.
\newblock In {\em International conference on machine learning}, pages 8748--8763. PmLR, 2021.

\bibitem{Shang2024:LLaVA-PruMerge}
Yuzhang Shang, Mu~Cai, Bingxin Xu, Yong~Jae Lee, and Yan Yan.
\newblock Llava-prumerge: Adaptive token reduction for efficient large multimodal models.
\newblock {\em arXiv preprint arXiv:2403.15388}, 2024.

\bibitem{singh2019towards}
Amanpreet Singh, Vivek Natarjan, Meet Shah, Yu~Jiang, Xinlei Chen, Devi Parikh, and Marcus Rohrbach.
\newblock Towards {VQA} models that can read.
\newblock In {\em Proceedings of the IEEE Conference on Computer Vision and Pattern Recognition}, pages 8317--8326, 2019.

\bibitem{tan2025tokencarve}
Xudong Tan, Peng Ye, Chongjun Tu, Jianjian Cao, Yaoxin Yang, Lin Zhang, Dongzhan Zhou, and Tao Chen.
\newblock Tokencarve: Information-preserving visual token compression in multimodal large language models.
\newblock {\em arXiv preprint arXiv:2503.10501}, 2025.

\bibitem{tang2023video}
Yunlong Tang, Jing Bi, Siting Xu, Luchuan Song, Susan Liang, Teng Wang, Daoan Zhang, Jie An, Jingyang Lin, Rongyi Zhu, et~al.
\newblock Video understanding with large language models: A survey.
\newblock {\em arXiv preprint arXiv:2312.17432}, 2023.

\bibitem{touvron2023llama}
Hugo Touvron, Louis Martin, Kevin Stone, Peter Albert, Amjad Almahairi, Yasmine Babaei, Nikolay Bashlykov, Soumya Batra, Prajjwal Bhargava, Shruti Bhosale, et~al.
\newblock Llama 2: Open foundation and fine-tuned chat models.
\newblock {\em arXiv preprint arXiv:2307.09288}, 2023.

\bibitem{tschannen2025siglip}
Michael Tschannen, Alexey Gritsenko, Xiao Wang, Muhammad~Ferjad Naeem, Ibrahim Alabdulmohsin, Nikhil Parthasarathy, Talfan Evans, Lucas Beyer, Ye~Xia, Basil Mustafa, et~al.
\newblock Siglip 2: Multilingual vision-language encoders with improved semantic understanding, localization, and dense features.
\newblock {\em arXiv preprint arXiv:2502.14786}, 2025.

\bibitem{vaswani2017attention}
Ashish Vaswani, Noam Shazeer, Niki Parmar, Jakob Uszkoreit, Llion Jones, Aidan~N Gomez, {\L}ukasz Kaiser, and Illia Polosukhin.
\newblock Attention is all you need.
\newblock {\em Advances in neural information processing systems}, 30, 2017.

\bibitem{wang2025folder}
Haicheng Wang, Zhemeng Yu, Gabriele Spadaro, Chen Ju, Victor Qu{\'e}tu, and Enzo Tartaglione.
\newblock Folder: Accelerating multi-modal large language models with enhanced performance.
\newblock {\em arXiv preprint arXiv:2501.02430}, 2025.

\bibitem{wang2024qwen2}
Peng Wang, Shuai Bai, Sinan Tan, Shijie Wang, Zhihao Fan, Jinze Bai, Keqin Chen, Xuejing Liu, Jialin Wang, Wenbin Ge, et~al.
\newblock Qwen2-vl: Enhancing vision-language model's perception of the world at any resolution.
\newblock {\em arXiv preprint arXiv:2409.12191}, 2024.

\bibitem{wang2025data}
Shaobo Wang, Xiangqi Jin, Ziming Wang, Jize Wang, Jiajun Zhang, Kaixin Li, Zichen Wen, Zhong Li, Conghui He, Xuming Hu, et~al.
\newblock Data whisperer: Efficient data selection for task-specific llm fine-tuning via few-shot in-context learning.
\newblock {\em arXiv preprint arXiv:2505.12212}, 2025.

\bibitem{wang2025winningpruninggambleunified}
Shaobo Wang, Jiaming Wang, Jiajun Zhang, Cong Wang, Yue Min, Zichen Wen, Fei Huang, Huiqiang Jiang, Junyang Lin, Dayiheng Liu, and Linfeng Zhang.
\newblock Winning the pruning gamble: A unified approach to joint sample and token pruning for efficient supervised fine-tuning.
\newblock {\em arXiv preprint arXiv:2509.23873}, 2025.

\bibitem{wang2024internvideo2}
Yi~Wang, Kunchang Li, Xinhao Li, Jiashuo Yu, Yinan He, Guo Chen, Baoqi Pei, Rongkun Zheng, Jilan Xu, Zun Wang, et~al.
\newblock Internvideo2: Scaling video foundation models for multimodal video understanding.
\newblock {\em Arxiv e-prints}, pages arXiv--2403, 2024.

\bibitem{wang2024exploring}
Yiqi Wang, Wentao Chen, Xiaotian Han, Xudong Lin, Haiteng Zhao, Yongfei Liu, Bohan Zhai, Jianbo Yuan, Quanzeng You, and Hongxia Yang.
\newblock Exploring the reasoning abilities of multimodal large language models (mllms): A comprehensive survey on emerging trends in multimodal reasoning.
\newblock {\em arXiv preprint arXiv:2401.06805}, 2024.

\bibitem{wen2025token}
Zichen Wen, Yifeng Gao, Weijia Li, Conghui He, and Linfeng Zhang.
\newblock Token pruning in multimodal large language models: Are we solving the right problem?
\newblock {\em arXiv preprint arXiv:2502.11501}, 2025.

\bibitem{wen2025stop}
Zichen Wen, Yifeng Gao, Shaobo Wang, Junyuan Zhang, Qintong Zhang, Weijia Li, Conghui He, and Linfeng Zhang.
\newblock Stop looking for important tokens in multimodal language models: Duplication matters more.
\newblock {\em arXiv preprint arXiv:2502.11494}, 2025.

\bibitem{xiong2025prune2drive}
Minhao Xiong, Zichen Wen, Zhuangcheng Gu, Xuyang Liu, Rui Zhang, Hengrui Kang, Jiabing Yang, Junyuan Zhang, Weijia Li, Conghui He, et~al.
\newblock Prune2drive: A plug-and-play framework for accelerating vision-language models in autonomous driving.
\newblock {\em arXiv preprint arXiv:2508.13305}, 2025.

\bibitem{yang2024qwen2}
An~Yang, Baosong Yang, Beichen Zhang, Binyuan Hui, Bo~Zheng, Bowen Yu, Chengyuan Li, Dayiheng Liu, Fei Huang, Haoran Wei, et~al.
\newblock Qwen2. 5 technical report.
\newblock {\em arXiv preprint arXiv:2412.15115}, 2024.

\bibitem{yang2025efficientvla}
Yantai Yang, Yuhao Wang, Zichen Wen, Luo Zhongwei, Chang Zou, Zhipeng Zhang, Chuan Wen, and Linfeng Zhang.
\newblock Efficientvla: Training-free acceleration and compression for vision-language-action models.
\newblock {\em arXiv preprint arXiv:2506.10100}, 2025.

\bibitem{ye2024mplug}
Qinghao Ye, Haiyang Xu, Jiabo Ye, Ming Yan, Anwen Hu, Haowei Liu, Qi~Qian, Ji~Zhang, and Fei Huang.
\newblock mplug-owl2: Revolutionizing multi-modal large language model with modality collaboration.
\newblock In {\em Proceedings of the ieee/cvf conference on computer vision and pattern recognition}, pages 13040--13051, 2024.

\bibitem{calflops}
Xiaoju Ye.
\newblock calflops: a flops and params calculate tool for neural networks in pytorch framework, 2023.

\bibitem{ye2024voco}
Xubing Ye, Yukang Gan, Xiaoke Huang, Yixiao Ge, and Yansong Tang.
\newblock Voco-llama: Towards vision compression with large language models.
\newblock {\em arXiv preprint arXiv:2406.12275}, 2024.

\bibitem{yin2023survey}
Shukang Yin, Chaoyou Fu, Sirui Zhao, Ke~Li, Xing Sun, Tong Xu, and Enhong Chen.
\newblock A survey on multimodal large language models.
\newblock {\em arXiv preprint arXiv:2306.13549}, 2023.

\bibitem{yuan2024llm}
Zhihang Yuan, Yuzhang Shang, Yang Zhou, Zhen Dong, Chenhao Xue, Bingzhe Wu, Zhikai Li, Qingyi Gu, Yong~Jae Lee, Yan Yan, Beidi Chen, Guangyu Sun, and Kurt Keutzer.
\newblock Llm inference unveiled: Survey and roofline model insights, 2024.

\bibitem{zhai2023sigmoid}
Xiaohua Zhai, Basil Mustafa, Alexander Kolesnikov, and Lucas Beyer.
\newblock Sigmoid loss for language image pre-training.
\newblock In {\em Proceedings of the IEEE/CVF international conference on computer vision}, pages 11975--11986, 2023.

\bibitem{zhang2024exploring}
Jiarui Zhang, Jinyi Hu, Mahyar Khayatkhoei, Filip Ilievski, and Maosong Sun.
\newblock Exploring perceptual limitation of multimodal large language models.
\newblock {\em arXiv preprint arXiv:2402.07384}, 2024.

\bibitem{zhang2023visual}
Jiarui Zhang, Mahyar Khayatkhoei, Prateek Chhikara, and Filip Ilievski.
\newblock Visual cropping improves zero-shot question answering of multimodal large language models.
\newblock In {\em R0-FoMo: Robustness of Few-shot and Zero-shot Learning in Large Foundation Models}, 2023.

\bibitem{zhang2025llava}
Shaolei Zhang, Qingkai Fang, Zhe Yang, and Yang Feng.
\newblock Llava-mini: Efficient image and video large multimodal models with one vision token.
\newblock {\em arXiv preprint arXiv:2501.03895}, 2025.

\bibitem{zhang2024beyond}
Yiming Zhang, Zhuokai Zhao, Zhaorun Chen, Zenghui Ding, Xianjun Yang, and Yining Sun.
\newblock Beyond training: Dynamic token merging for zero-shot video understanding.
\newblock {\em arXiv preprint arXiv:2411.14401}, 2024.

\bibitem{zhang2024sparsevlm}
Yuan Zhang, Chun-Kai Fan, Junpeng Ma, Wenzhao Zheng, Tao Huang, Kuan Cheng, Denis Gudovskiy, Tomoyuki Okuno, Yohei Nakata, Kurt Keutzer, et~al.
\newblock Sparsevlm: Visual token sparsification for efficient vision-language model inference.
\newblock {\em arXiv preprint arXiv:2410.04417}, 2024.

\bibitem{zhang2024treat}
Zeliang Zhang, Phu Pham, Wentian Zhao, Kun Wan, Yu-Jhe Li, Jianing Zhou, Daniel Miranda, Ajinkya Kale, and Chenliang Xu.
\newblock Treat visual tokens as text? but your mllm only needs fewer efforts to see.
\newblock {\em arXiv preprint arXiv:2410.06169}, 2024.

\bibitem{zhao2024accelerating}
Shiyu Zhao, Zhenting Wang, Felix Juefei-Xu, Xide Xia, Miao Liu, Xiaofang Wang, Mingfu Liang, Ning Zhang, Dimitris~N Metaxas, and Licheng Yu.
\newblock Accelerating multimodel large language models by searching optimal vision token reduction.
\newblock {\em arXiv preprint arXiv:2412.00556}, 2024.

\bibitem{zhu2023minigpt}
Deyao Zhu, Jun Chen, Xiaoqian Shen, Xiang Li, and Mohamed Elhoseiny.
\newblock Minigpt-4: Enhancing vision-language understanding with advanced large language models.
\newblock {\em arXiv preprint arXiv:2304.10592}, 2023.

\end{thebibliography}
% {
%     \small
%     \bibliographystyle{plain}
%     \bibliography{main}
% }

%%%%%%%%%%%%%%%%%%%%%%%%%%%%%%%%%%%%%%%%%%%%%%%%%%%%%%%%%%%%

% \clearpage
%%%%%%%%%%%%%%%%%%%%%%%%%%%%%%%%%%%%%%%%%%%%%%%%%%%%%%%%%%%%

% \newpage
\section*{NeurIPS Paper Checklist}

\begin{enumerate}

\item {\bf Claims}
    \item[] Question: Do the main claims made in the abstract and introduction accurately reflect the paper's contributions and scope?
    \item[] Answer: \answerYes{} % Replace by \answerYes{}, \answerNo{}, or \answerNA{}.
    \item[] Justification: The main contributions and scope of this paper are outlined in the abstract and introduction sections.

\item {\bf Limitations}
    \item[] Question: Does the paper discuss the limitations of the work performed by the authors?
    \item[] Answer: \answerYes{} % Replace by \answerYes{}, \answerNo{}, or \answerNA{}.
    \item[] Justification: In Appendix~\ref{app_sec:future_works}, we outline the potential limitations of this work. 

\item {\bf Theory assumptions and proofs}
    \item[] Question: For each theoretical result, does the paper provide the full set of assumptions and a complete (and correct) proof?
    \item[] Answer: \answerYes{} % Replace by \answerYes{}, \answerNo{}, or \answerNA{}.
    \item[] Justification: In the main paper and Appendix, we provide the full set of assumptions and complete proof.

    \item {\bf Experimental result reproducibility}
    \item[] Question: Does the paper fully disclose all the information needed to reproduce the main experimental results of the paper to the extent that it affects the main claims and/or conclusions of the paper (regardless of whether the code and data are provided or not)?
    \item[] Answer: \answerYes{} % Replace by \answerYes{}, \answerNo{}, or \answerNA{}.
    \item[] Justification: We provide details of all model versions, dataset versions, hyperparameters, and experimental settings in Sec.~\ref{sec:experimental_setting}, 
 and Appendix~\ref{app_sec:train_detail}.

\item {\bf Open access to data and code}
    \item[] Question: Does the paper provide open access to the data and code, with sufficient instructions to faithfully reproduce the main experimental results, as described in supplemental material?
    \item[] Answer: \answerNo{} % Replace by \answerYes{}, \answerNo{}, or \answerNA{}.
    \item[] Justification: In accordance with the requirements of the supporting organization, the code will be released after the review process is completed. We are committed to providing sufficient instructions to ensure reproducibility at that time.

\item {\bf Experimental setting/details}
    \item[] Question: Does the paper specify all the training and test details (e.g., data splits, hyperparameters, how they were chosen, type of optimizer, etc.) necessary to understand the results?
    \item[] Answer: \answerYes{} % Replace by \answerYes{}, \answerNo{}, or \answerNA{}.
    \item[] Justification: All the training and evaluation details have been provided in Sec.~\ref{sec:experimental_setting} and Appendix~\ref{app_sec:train_detail}.   

\item {\bf Experiment statistical significance}
    \item[] Question: Does the paper report error bars suitably and correctly defined or other appropriate information about the statistical significance of the experiments?
    \item[] Answer: \answerYes{} % Replace by \answerYes{}, \answerNo{}, or \answerNA{}.
    \item[] Justification: Some experiments, including latency statistics, were conducted multiple times, and the corresponding error bars or bands are reported in Sec.~\ref{sec:analyses}.

\item {\bf Experiments compute resources}
    \item[] Question: For each experiment, does the paper provide sufficient information on the computer resources (type of compute workers, memory, time of execution) needed to reproduce the experiments?
    \item[] Answer: \answerYes{} % Replace by \answerYes{}, \answerNo{}, or \answerNA{}.
    \item[] Justification: We include descriptions of the computational resources used in Sec.~\ref{sec: efficiency} and Appedix~\ref{app_sec:train_detail}.
    
\item {\bf Code of ethics}
    \item[] Question: Does the research conducted in the paper conform, in every respect, with the NeurIPS Code of Ethics \url{https://neurips.cc/public/EthicsGuidelines}?
    \item[] Answer: \answerYes{} % Replace by \answerYes{}, \answerNo{}, or \answerNA{}.
    \item[] Justification: We followed the NeurIPS Code of Ethics.

\item {\bf Broader impacts}
    \item[] Question: Does the paper discuss both potential positive societal impacts and negative societal impacts of the work performed?
    \item[] Answer: \answerYes{} % Replace by \answerYes{}, \answerNo{}, or \answerNA{}.
    \item[] Justification: We describe broader impacts of our work in Appendix~\ref{app_sec:broader_impact}.
    
\item {\bf Safeguards}
    \item[] Question: Does the paper describe safeguards that have been put in place for responsible release of data or models that have a high risk for misuse (e.g., pretrained language models, image generators, or scraped datasets)?
    \item[] Answer: \answerNA{} % Replace by \answerYes{}, \answerNo{}, or \answerNA{}.
    \item[] Justification: The paper poses no such risks.

\item {\bf Licenses for existing assets}
    \item[] Question: Are the creators or original owners of assets (e.g., code, data, models), used in the paper, properly credited and are the license and terms of use explicitly mentioned and properly respected?
    \item[] Answer: \answerYes{} % Replace by \answerYes{}, \answerNo{}, or \answerNA{}.
    \item[] Justification: The creators or original owners of assets (e.g., code, data, models) used in the paper have been properly credited, and the corresponding papers have been cited.

\item {\bf New assets}
    \item[] Question: Are new assets introduced in the paper well documented and is the documentation provided alongside the assets?
    \item[] Answer: \answerNA{} % Replace by \answerYes{}, \answerNo{}, or \answerNA{}.
    \item[] Justification: The paper does not release new assets.

\item {\bf Crowdsourcing and research with human subjects}
    \item[] Question: For crowdsourcing experiments and research with human subjects, does the paper include the full text of instructions given to participants and screenshots, if applicable, as well as details about compensation (if any)? 
    \item[] Answer: \answerNA{} % Replace by \answerYes{}, \answerNo{}, or \answerNA{}.
    \item[] Justification: The paper does not involve crowdsourcing nor research with human subjects.

\item {\bf Institutional review board (IRB) approvals or equivalent for research with human subjects}
    \item[] Question: Does the paper describe potential risks incurred by study participants, whether such risks were disclosed to the subjects, and whether Institutional Review Board (IRB) approvals (or an equivalent approval/review based on the requirements of your country or institution) were obtained?
    \item[] Answer: \answerNA{} % Replace by \answerYes{}, \answerNo{}, or \answerNA{}.
    \item[] Justification: The paper does not involve crowdsourcing nor research with human subjects.

\item {\bf Declaration of LLM usage}
    \item[] Question: Does the paper describe the usage of LLMs if it is an important, original, or non-standard component of the core methods in this research? Note that if the LLM is used only for writing, editing, or formatting purposes and does not impact the core methodology, scientific rigorousness, or originality of the research, declaration is not required.
    %this research? 
    \item[] Answer: \answerNA{} % Replace by \answerYes{}, \answerNo{}, or \answerNA{}.
    \item[] Justification: The core method development in this research does not involve LLMs as any important, original, or non-standard components.

\end{enumerate}

\clearpage
\appendix
\startcontents[appendix]
\printcontents[appendix]{ }{0}{\section*{Appendix}}

\section{More Generalizability Experiments}
\label{app_sec:more_generalizability_experiments}
\begin{figure}[!ht]
    \centering
    \includegraphics[width=1.0\linewidth]{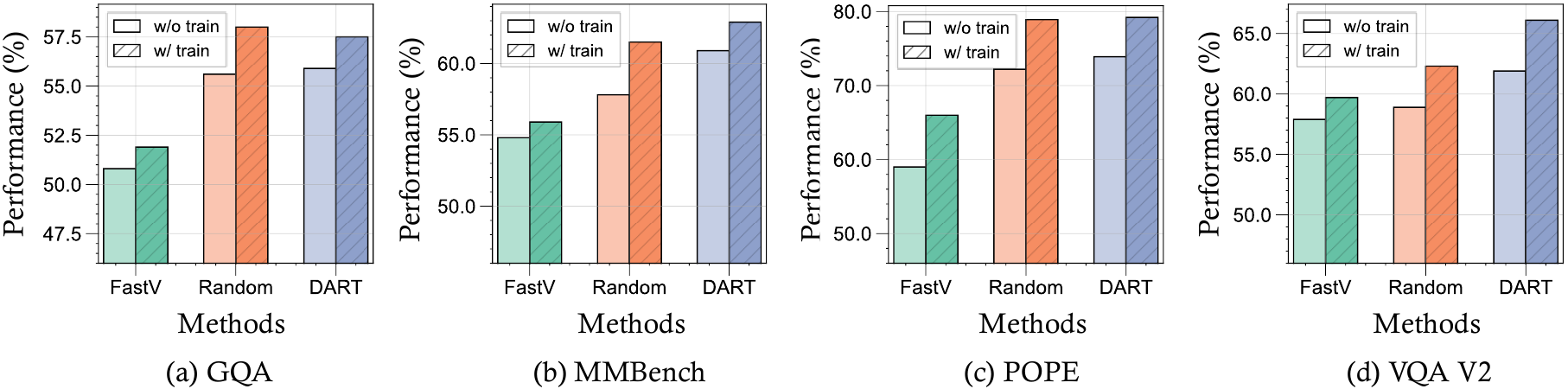}
    % \vspace{-7mm}
    \caption{Following model architecture and training data of LLaVA-v1.5-7B, we apply the \textbf{Random token compression} for token consistency distillation during training. Here, ``w/o train'' denotes the vanilla LLaVA. During inference, all methods achieve a token compression ratio of 88.9\%.}
    % \vspace{-4mm}
    \label{fig:generalized_exp_random}
\end{figure}

\begin{figure}[!ht]
    \centering
    \includegraphics[width=1.0\linewidth]{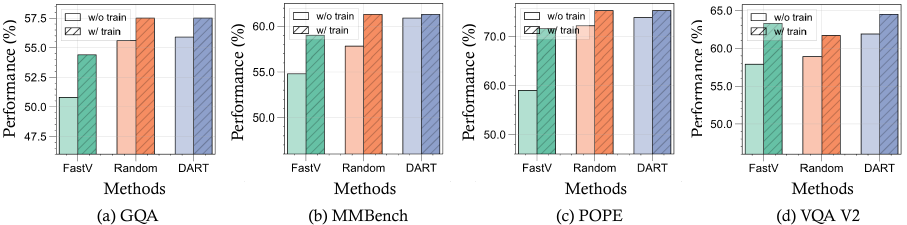}
    % \vspace{-7mm}
    \caption{Following model architecture and training data of LLaVA-v1.5-7B, we apply the \textbf{FastV} for token consistency distillation during training. Here, ``w/o train'' denotes the vanilla LLaVA. During inference, all methods achieve a token compression ratio of 88.9\%.}
    % \vspace{-4mm}
    \label{fig:generalized_exp_fastv}
\end{figure}
In addition to the cross-method generalization validation of our proposed framework conducted in Sec.~\ref{sec:generalize}, we performed more comprehensive and meticulous experiments to further verify its effectiveness.
Specifically, as illustrated in Figures~\ref{fig:generalized_exp_random} and \ref{fig:generalized_exp_fastv}, we employ Random Token Compression and FastV, respectively, to perform token consistency distillation on the multi-modal large language model.
Experimental results demonstrate that regardless of the token compression strategy used during training with our proposed method, the trained model consistently achieves improved inference performance across various token compression strategies.
This strongly demonstrates the generalizability and effectiveness of our approach, indicating that models trained under our framework genuinely adapt to the pattern of missing visual tokens without being constrained to a fixed set of preserved tokens.

\section{Detailed Theoretical Analysis}\label{app_sec:proof}

\begin{proof}[Proof of Theorem~\ref{thm:scalar}] 
\textbf{1. Direct path.}
Since \(\theta^{\mathrm{dir}}_{r_t} = c(r_t)\), the total variation is
$$
\TV(\{\theta^{\mathrm{dir}}_{r_t}\}) = \sum_{t=0}^{T-1} (c(r_{t+1}) - c(r_t)).
$$
By monotonicity \textbf{(S1)} and Lipschitz continuity \textbf{(S2)}, each term satisfies
$$
c(r_{t+1}) - c(r_t) \le \gamma (r_{t+1} - r_t),
$$
so summing yields:
$$
\TV(\{\theta^{\mathrm{dir}}_{r_t}\}) \le \gamma \sum_{t=0}^{T-1} (r_{t+1} - r_t) = \gamma r_{\max}.
$$

\textbf{2. Progressive path.}
From the closed-form minimizer,
$$
\theta^{\mathrm{prog}}_{r_t} = \frac{c(r_t) + \lambda\,c(r_t - \Delta)}{1 + \lambda},
$$
we have:
$$
\theta^{\mathrm{prog}}_{r_{t+1}} - \theta^{\mathrm{prog}}_{r_t}
= \frac{(c(r_{t+1}) - c(r_t)) + \lambda (c(r_{t+1} - \Delta) - c(r_t - \Delta))}{1 + \lambda}.
$$
Define \(\Delta_t := c(r_{t+1}) - c(r_t) \ge 0\) and \(\tilde{\Delta}_t := c(r_{t+1} - \Delta) - c(r_t - \Delta) \ge 0\). Then:
$$
\left| \theta^{\mathrm{prog}}_{r_{t+1}} - \theta^{\mathrm{prog}}_{r_t} \right|
\le \frac{\Delta_t + \lambda \tilde{\Delta}_t}{1 + \lambda}
= \frac{1 + \lambda \kappa_t}{1 + \lambda} \cdot \Delta_t,
\quad \text{where } \kappa_t := \frac{\tilde{\Delta}_t}{\Delta_t}.
$$

\textbf{3. Bounding \(\kappa_t\).}
Define the worst-case slope ratio:
$$
  \kappa := \sup_{r \in [\Delta,\, r_{\max} - \Delta]}
  \frac{c(r) - c(r - \Delta)}{c(r + \Delta) - c(r)}.
$$
Each \(\kappa_t\) compares the lagged slope to the current slope. If the schedule is appropriately spaced so that \(r_t \in [\Delta, r_{\max} - \Delta]\), then we have:
$$
\kappa_t = \frac{c(r_{t+1} - \Delta) - c(r_t - \Delta)}{c(r_{t+1}) - c(r_t)} \le \kappa.
$$
Convexity \textbf{(S3)} ensures that \(c'(r)\) is non-decreasing, so
$$
c(r) - c(r - \Delta) \le c(r + \Delta) - c(r)
\quad \Rightarrow \quad
\kappa \le 1.
$$
If \(c\) is not affine (i.e., \(c''(r) > 0\) on a set of positive measure), the inequality is strict for some \(r\), and thus \(\kappa < 1\).

\textbf{4. Total variation bound.}
Summing over steps gives:
$$
  \TV(\{\theta^{\mathrm{prog}}_{r_t}\})
  = \sum_{t=0}^{T-1} \left| \theta^{\mathrm{prog}}_{r_{t+1}} - \theta^{\mathrm{prog}}_{r_t} \right|
  \le \sum_{t=0}^{T-1} \frac{1 + \lambda \kappa}{1 + \lambda} \cdot \Delta_t
  = \frac{1 + \lambda \kappa}{1 + \lambda} \cdot \TV(\{\theta^{\mathrm{dir}}_{r_t}\}),
$$
with strict inequality when \(\kappa < 1\), completing the proof.
\end{proof}

\section{Integrated Progressive Consistency Distillation}
\label{app_sec:Integrated_progressive_distillation}
\definecolor{mygray}{gray}{.92}
\begin{table*}[!ht]
\caption{Performance on 10 visual understanding benchmarks. `Res.' is resolution, and `\#Vision Tokens' is the number of vision tokens fed to the LLM backbone. Both training and inference of \textbf{Integrated Progressive Consistency Distillation (ICD)} employ DART as the token compression strategy for our methods. Parentheses in Avg.(\%) column show diffs vs. LLaVA-v1.5.}
% \vspace{1mm}
\label{tab:app_integrated_result}
\centering
\tiny
\begin{tabular}%{L{1.56cm}L{1.12cm}c|c|C{0.27cm}C{0.27cm}C{0.27cm}C{0.27cm}C{0.27cm}C{0.27cm}C{0.38cm}C{0.27cm}C{0.27cm}C{0.38cm}|C{0.27cm}} 
% {L{1.90cm}L{1.12cm}c|c|C{0.3cm}C{0.3cm}C{0.3cm}C{0.3cm}C{0.3cm}C{0.3cm}C{0.38cm}C{0.3cm}C{0.3cm}C{0.38cm}|C{0.3cm}} 
{L{1.75cm}L{1.12cm}C{0.2cm}|C{0.55cm}|C{0.3cm}C{0.3cm}C{0.3cm}C{0.3cm}C{0.3cm}C{0.3cm}C{0.38cm}C{0.3cm}C{0.3cm}C{0.38cm}|C{0.84cm}} 
\toprule
\textbf{Methods} & \textbf{LLM} & \textbf{Res.} & \textbf{\begin{tabular}[c]{@{}c@{}}\#Vision\\ Tokens\end{tabular}} & $\!\!\textbf{\text{VQA}}^{\!\text{V2}}$ & $\!\!$\textbf{GQA} & $\!\!\!$\textbf{VizWiz} & $\!\!$$\textbf{\text{SQA}}^{\!\text{I}}$ & $\!\!$$\textbf{\text{VQA}}^{\!\text{T}}$ & $\!\!$\textbf{POPE} & \textbf{MME} & $\!\!$\textbf{MMB} & $\!$\textbf{\begin{tabular}[c]{@{}c@{}}MMB-\\ CN\end{tabular}} & $\!$\textbf{\begin{tabular}[c]{@{}c@{}}OCR\\ Bench\end{tabular}} & \textbf{\begin{tabular}[c]{@{}c@{}}Avg.\\ (\%)\end{tabular}} \\
\midrule
\textbf{BLIP-2}~\citep{li2023blip} & Vicuna-13B & 224 & 32 & 65.0 & 41.0 & 19.6 & 61.0 & 42.5 & 85.3 & -- & – & -- & -- & -- \\
\textbf{InstructBLIP}~\citep{instructblip} & Vicuna-7B & 224 & 32 & 66.3 & 49.2 & 34.5 & 60.5 & 50.1 & 83.9 & 1500 & 36.0 & -- & 259 & -- \\
\textbf{InstructBLIP}~\citep{instructblip} & Vicuna-13B & 224 & 32 & 64.2 & 49.5 & 33.4 & 63.1 & -- & 84.1 & 1530 & 36.9 & 17.4 & 252 & --  \\
\textbf{IDEFICS-9B}~\citep{laurenccon2023obelics} & LLaMA-7B & 224 & 64 & 50.9 & 38.4 & 35.5 & – & 25.9 & 75.3 & 1027 & 48.2 & -- & 245 & -- \\
\textbf{IDEFICS-80B}~\citep{laurenccon2023obelics} & LLaMA-65B & 224 & 64 & 60.0 & 45.2 & 36.0 & -- & 30.9 & -- & 1076 & 54.5 & 29.1 & 277 & -- \\
\textbf{Qwen-VL}~\citep{Bai:Qwen-VL} & Qwen-7B & 448 & 256 & -- & 59.3 & 35.2 & 67.1 & 63.8 & -- & 1708 & 38.2 & -- & 133 & -- \\
\textbf{Qwen-VL-Chat}~\citep{Bai:Qwen-VL} & Qwen-7B & 448 & 256 & -- & 57.5 & 38.9 & 68.2 & 61.5 & -- & 1891 & 60.6 & -- & 267 & -- \\
\textbf{SPHINX}~\citep{lin2023sphinx} & LLaMA-13B & 224 & 289 & 78.1 & 62.6 & 39.9 & 69.3 & 51.6 & 80.7 & -- & 66.9 & -- & -- & -- \\
\textbf{SPHINX-2k}~\citep{lin2023sphinx} & LLaMA-13B & 762 & 2890 & 80.7 & 63.1 & 44.9 & 70.6 & 61.2 & 87.2 & -- & 65.9 & -- & -- & -- \\
\textbf{mPLUG-Owl2}~\citep{ye2024mplug} & LLaMA-7B & 448 & 1024 & 79.4 & 56.1 & 54.5 & 68.7 & 54.3 & - & -- & 64.5 & -- & -- & -- \\
\textbf{Video-LLaVA}~\citep{lin2023video} & Vicuna-7B & 224 & 256 & 65.9 & 60.3 & 48.1 & 66.4 & 51.8 & 83.1 & 1542 & 60.6 & 49.3 & 161 & 55.7 \\
\textbf{LLaVA-v1.5}~\citep{liu2023llava} & Vicuna-7B & 336 & 576 & 72.2 & 61.9 & 52.5 & 68.3 & 58.1 & 85.9 & 1785 & 64.1 & 55.8 & 319 & 61.4 \\
% \textbf{LLaVA-v1.5} & Vicuna-13B & 336 & 576 & --  \\
% \textbf{LLaVA-v1.6} & -- \\
\rowcolor{mygray}\multicolumn{15}{c}{\scriptsize\textit{\textbf{LMMs with fewer vision tokens}}} \\
\textbf{Average-Pooling} & Vicuna-7B & 336 & 64 & 63.0 & 55.5 & 48.4 & 68.6 & 52.6 & 79.2 & 1579 & 59.6 & 49.5 & 258 & 55.9 {\color[HTML]{C2183A} \fontsize{5pt}{6pt}\selectfont (-5.5)}    \\
\textbf{MQT-LLaVA}~\citep{hu2024matryoshka} & Vicuna-7B & 336 & 2 & 51.4 & 49.6 & 50.0 & 66.1 & 14.8 & 75.4 & 1402 & 48.9 & 40.5 & 169 & 46.4 {\color[HTML]{C2183A} \fontsize{5pt}{6pt}\selectfont (-15)} \\
\textbf{MQT-LLaVA}~\citep{hu2024matryoshka} & Vicuna-7B & 336 & 36 & 62.0 & 57.7 & 53.6 & 69.2 & 28.6 & 82.9 & 1777 & 60.5 & 51.6 & 244 & 55.4 {\color[HTML]{C2183A} \fontsize{5pt}{6pt}\selectfont (-6.0)} \\
\textbf{MQT-LLaVA}~\citep{hu2024matryoshka} & Vicuna-7B & 336 & 64 & 65.6 & 58.7 & 54.3 & 68.4 & 32.5 & 83.1 & 1810 & 61.3 & 53.7 & 260 & 56.8 {\color[HTML]{C2183A} \fontsize{5pt}{6pt}\selectfont (-4.6)}  \\
\textbf{MQT-LLaVA}~\citep{hu2024matryoshka} & Vicuna-7B & 336 & 128 & 66.2 & 59.8 & 54.6 & 69.3 & 35.7 & 84.3 & 1773 & 62.1 & 53.6 & 266 & 57.6 {\color[HTML]{C2183A} \fontsize{5pt}{6pt}\selectfont (-3.8)} \\
\textbf{MQT-LLaVA}~\citep{hu2024matryoshka} & Vicuna-7B & 336 & 192 & 66.9 & 59.9 & 54.6 & 69.1 & 35.8 & 85.1 & 1784 & 62.0 & 53.9 & 263 & 57.7 {\color[HTML]{C2183A} \fontsize{5pt}{6pt}\selectfont (-3.7)} \\
\textbf{MQT-LLaVA}~\citep{hu2024matryoshka} & Vicuna-7B & 336 & 256 & 68.3 & 60.1 & 54.6 & 69.0 & 37.1 & 84.6 & 1740 & 61.7 & 53.0 & 273 & 57.8 {\color[HTML]{C2183A} \fontsize{5pt}{6pt}\selectfont (-3.6)} \\
\textbf{QT-LLaVA} & Vicuna-7B & 336 & 256 & -- & 60.3 & 51.5 & 68.1 & 36.9 & 84.1 & 1771 & 62.1 & 53.9 & 265 & --  \\
% \textbf{PruMerge} & Vicuna-7B & 336 & 64 & 67.4 & 51.9 & 50.1 & 68.1 & 54.0 & 65.3 & -- & 55.3 & 49.1 & 250 & --   \\
% \textbf{PruMerge++} & Vicuna-7B & 336 & 144 & -- & -- & -- & -- & -- & -- & -- & -- & -- & -- & -- \\
\textbf{LLaMA-VID}~\citep{li2024llama} & Vicuna-7B & 336 & 2 & -- & 55.5 & 54.2 & 68.8 & 49.0 & 83.1 & -- & -- & -- & -- & --  \\
\textbf{VoCo-LLaMA}~\citep{ye2024voco} & Vicuna-7B & 336 & 1 & -- & 55.6 & 54.6 & 68.4 & 31.7 & 80.8 & 1594 & 56.4 & 46.2 & 69 & -- \\
\textbf{TokenPacker}~\citep{li2024tokenpacker} & Vicuna-7B & 336 & 144 & 71.3 & 62.0 & 56.6 & 70.5 & 43.8 & 86.2 & 1716 & 63.9 & 53.4 & 303 & 59.9 {\color[HTML]{C2183A} \fontsize{5pt}{6pt}\selectfont (-1.5)} \\
\textbf{TokenPacker}~\citep{li2024tokenpacker} & Vicuna-7B & 336 & 36 & -- & 58.6 & 50.2 & -- & -- & 83.7 & -- & 62.8 & -- & -- & --   \\
\textbf{LLaVA-Mini}~\citep{zhang2025llava} & Vicuna-7B & 336 & 144 & 58.1 & 56.3 & 14.8 & 25.3 & 26.0 & 82.3 & 1325 & 24.8 & -- & 132 & -- \\
\textbf{LLaVA-Mini}~\citep{zhang2025llava} & Vicuna-7B & 336 & 64 & - & 56.6 & 10.4 & 27.4 & 28.1 & 82.3 & 1324 & 23.8 & -- & 145 & --   \\
\rowcolor{mygray}\multicolumn{15}{c}{\scriptsize\textit{\textbf{Ours}}} \\
% \rowcolor{green!7}
\textbf{LLaVA-v1.5 + ICD} & Vicuna-7B & 336 & 256 & 72.2 & 61.7 & 54.7 & 70.0 & 58.1 & 85.7 & 1788 & 64.9 & 55.7 & 312 & 61.8 {\color[HTML]{18A6C2} \fontsize{5pt}{6pt}\selectfont (+0.4)}  \\ 
% \rowcolor{green!7}
\textbf{LLaVA-v1.5 + ICD} & Vicuna-7B & 336 & 192 & 71.9 & 61.4 & 54.8 & 70.3 & 57.6 & 84.8 & 1783 & 64.8 & 55.9 & 316 & 61.7 {\color[HTML]{18A6C2} \fontsize{5pt}{6pt}\selectfont (+0.3)}  \\ 
% \rowcolor{green!7}
\textbf{LLaVA-v1.5 + ICD} & Vicuna-7B & 336 & 128 & 69.6 & 60.3 & 55.2 & 70.1 & 57.1 & 83.5 & 1770 & 63.9 & 55.2 & 298 & 60.8 {\color[HTML]{C2183A} \fontsize{5pt}{6pt}\selectfont (-0.6)} \\ 
% \rowcolor{green!7}
\textbf{LLaVA-v1.5 + ICD} & Vicuna-7B & 336 & 64 & 66.5 & 57.5 & 55.3 & 70.4 & 54.7 & 79.3 & 1720 & 63.8 & 53.2 & 275 & 59.0 {\color[HTML]{C2183A} \fontsize{5pt}{6pt}\selectfont (-2.4)} \\ 
% \rowcolor{green!7}
\textbf{LLaVA-v1.5 + ICD} & Vicuna-7B & 336 & 36 & 62.1 & 55.6 & 54.4 & 70.9 & 52.8 & 74.0 & 1635 & 61.2 & 51.1 & 246 & 56.5 {\color[HTML]{C2183A} \fontsize{5pt}{6pt}\selectfont (-4.9)} \\ 
\bottomrule 
\end{tabular}
% \vspace{-6mm}
\end{table*}
% 192 & 71.9 & 61.4 & 54.8 & 70.3 & 57.6 & 84.8 & 1783 & 64.8 & 55.9 & 316
% 256 & 72.2 & 61.7 & 54.7 & 70.0 & 58.1 & 85.7 & 1788 & 64.9 & 55.7 & 312
% 128 & 69.6 & 60.3 & 55.2 & 70.1 & 57.1 & 83.5 & 1770 & 63.9 & 55.2 & 298 & --
% 36 & 62.1 & 55.6 & 54.4 & 70.9 & 52.8 & 74.0 & 1635 & 61.2 & 51.1 & 246 & --

As outlined in Sec.~\ref{sec:intro}, our Progressive Consistency Distillation framework incorporates two progressive learning mechanisms: Token Consistency Distillation (TCD) in a token-wise manner and Layer Consistency Distillation (LCD) in a layer-wise manner.
Experimental results demonstrate that both approaches achieve strong performance across all key metrics: model accuracy, efficiency, robustness, and generalization capability.
To further investigate this direction, we explore the integration of progressive learning from both token-wise and layer-wise perspectives. A critical question arises: \emph{Can this integrated approach continue to maintain or even enhance the model's performance?}
Building upon these two approaches, we design Integrated Progressive Consistency Distillation (ICD). Unlike the original Token Consistency Distillation (TCD), where token-wise progression is governed by global training progress, our method enables layer-wise iterative application of TCD, effectively integrating both token-wise and layer-wise dimensions. Specifically, during training, the token compression layer progressively shifts from deeper to shallower layers. Within each layer, the compression ratio follows the TCD schedule, sampling from small to large values. Upon transitioning to the next layer, the ratio resets to its initial value, and the process repeats.
As shown in Table~\ref{tab:app_integrated_result}, the Integrated Progressive Consistency Distillation approach achieves comparable performance on representative visual benchmarks. Notably, it even surpasses the vanilla LLaVA in average performance while retaining only 192 visual tokens ($66.7\% \downarrow$). 
Furthermore, when compared to other training-aware token compression approaches (such as MQT-LLaVA and TokenPacker), ICD demonstrates superior performance under identical retained visual tokens.

\section{Experimental Setup}
\label{app_sec:train_detail}
\subsection{Token Compression Techniques During Training}
We adopt three representative token compression methods in our proposed framework: FastV, which follows an importance-based strategy; DART, which leverages redundancy-based pruning; and random token pruning, the simplest form of token compression.
\begin{itemize}[leftmargin=*]
\item DART~\citep{wen2025stop} is a training-free and plug-and-play method that prunes visual tokens based on token duplication while maintaining compatibility with efficient attention mechanisms like Flash Attention~\citep{dao2022flashattention,dao2023flashattention}.

\item FastV~\citep{chen2024image} is a plug-and-play token compression technique that builds on the observation that visual tokens tend to have diminishing contributions to model outputs in deeper layers. By utilizing attention scores to assess token importance, it prunes less critical visual tokens at earlier layers of the language model.

\item Random is a token compression method that requires no additional signals or computation, and is primarily used to validate the effectiveness and generalizability of our proposed approach.
\end{itemize}

\begin{table}[!ht]
\centering
\caption{Method Comparison Across Different Stages}
\vspace{2mm}
\resizebox{0.65\linewidth}{!}{
\begin{tabular}{l|cccc}
\toprule
% \rowcolor{gray!20}
\textbf{Method} & \textbf{Stage 1} & \textbf{Stage 2} & \textbf{Stage 3} & \textbf{Training Time $\downarrow$ (h)} \\
\midrule
QT-LLaVA & $\checkmark$ & $\checkmark$ & & $\sim$ 30h $\times$ 8 A100s  \\
MQT-LLaVA & $\checkmark$ & $\checkmark$ & & $\sim$ 34h $\times$ 8 A6000s \\
LLaMA-VID & $\checkmark$ & $\checkmark$ & $\checkmark$ & $\sim$ 48h $\times$ 8 A100s \\
LLaVA-Mini & $\checkmark$ & $\checkmark$ & & $\sim$ 26h $\times$ 8 A100s \\
% \addlinespace[0.5em] 
\texttt{\algname} (Ours) & & $\checkmark$ & & $\sim$ 12.2h $\times$ 8 A100s \\
\bottomrule
\end{tabular}}
\label{tab:training_efficiency}
\end{table}

% \begin{table}[!ht]
% \centering
% \caption{Method Comparison Across Different Training Stages}
% \label{tab:training_efficiency}
% \resizebox{0.7\linewidth}{!}{
% \begin{tabular}{lcccc}
% \toprule
% \textbf{Method} & \textbf{Stage 1} & \textbf{Stage 2} & \textbf{Stage 3} & \textbf{Training Time} \\ 
% \cmidrule(lr){1-1} \cmidrule(lr){2-4} \cmidrule(lr){5-5}
% QT-LLaVA       & \textcolor{green}{\checkmark} & \textcolor{green}{\checkmark} & & \textcolor{gray}{--} \\
% MQT-LLaVA      & \textcolor{green}{\checkmark} & \textcolor{green}{\checkmark} & & \textcolor{gray}{--} \\
% LLaMA-VID      & \textcolor{green}{\checkmark} & \textcolor{green}{\checkmark} & \textcolor{green}{\checkmark} & \textcolor{gray}{--} \\
% LLaVA-Mini     & \textcolor{green}{\checkmark} & \textcolor{green}{\checkmark} & & \textcolor{gray}{--} \\
% \addlinespace[0.2em]
% \rowcolor{blue!5}
% \algname\ (Ours) & & \textcolor{green}{\checkmark} & & \textbf{12.2h} \\
% \bottomrule
% \end{tabular}}
% \end{table}
\subsection{Training Details}
\label{app_sec:training_details}
For a fair comparison, our framework not only adheres to the same model architecture and instruction-tuning data as vanilla LLaVA but also maintains identical hyperparameter settings. 
Moreover, since we have not modified the model architecture, we can directly use the pre-trained projector—unlike MQT-LLaVA, QT-LLaVA, and LLaVA-Mini, which require mandatory Stage 1 pre-training.
The training process, conducted on 8 $\times$ A100 GPUs, takes approximately 12 hours.
Furthermore, we faithfully reproduced the entire LLaVA training process following the official LLaVA-v1.5 training guidelines. All other token compression baselines are either trained following the settings provided in the original papers or evaluated using publicly available model checkpoints. 
For the ablation study in Section~\ref{sec:ablation}, the experiment without the distillation loss was trained using the same LLaVA-665K SFT data, with all other training parameters and procedures kept identical to those of \texttt{\algname}. For the variants without a progressive compression ratio (Table~\ref{tab:abalation_study1}) or progressive compression layer (Table~\ref{tab:abalation_study2}), we fixed the second layer as the compression layer. Overall, our ablation studies were conducted under a strictly fair and consistent experimental setup.
% -------------***************************--------------%
% \begin{table}[h]
% \centering
% \caption{
% Training settings of our proposed method.
% }
% \begin{tabular}{l| c c c}
% \toprule
% Settings & Stage 1 & Stage 2 & Stage 3 \\
% \midrule
% Batch size & 256 & 128 & 8 \\
% Learning rate & 1e-3 & 2e-5 & 2e-5 \\
% Learning schedule & \multicolumn{3}{c}{Cosine decay} \\
% Warmup ratio & \multicolumn{3}{c}{0.03} \\
% Weight decay & \multicolumn{3}{c}{0} \\
% Epoch & \multicolumn{3}{c}{1} \\
% Optimizer & \multicolumn{3}{c}{AdamW} \\
% DeepSpeed stage & \multicolumn{3}{c}{2} \\
% Vision encoder & \multicolumn{3}{c}{Freeze} \\
% Text decoder & Freeze & Open & Freeze \\
% Max token & 2048 & 2048 & 65536 \\
% \bottomrule
% \end{tabular}
% \label{tab:train_detail}
% \end{table}
% -------------***************************--------------%

\begin{table}[!ht]
\centering
\caption{Detailed hyperparameter settings.}
\vspace{2mm}
\resizebox{0.35\textwidth}{!}{
\begin{tabular}{l| c}
\toprule
\textbf{Settings} & \textbf{Stage 2} \\
\midrule
Batch size & 128 \\
Learning rate & 2e-5 \\
Learning schedule & Cosine decay \\
Warmup ratio & 0.03 \\
Weight decay & 0 \\
Epoch & 1 \\
Optimizer & AdamW \\
DeepSpeed stage & 3 \\
% Vision encoder & Freeze \\
% Text decoder & Open \\
Max token & 2048 \\
\bottomrule
\end{tabular}}
\label{tab:train_detail}
\end{table}

% \begin{wraptable}{r}{0.45\textwidth}  % r 表示右侧，宽度为 0.35\textwidth
% \centering
% \caption{Detailed hyperparameter settings.}
% \resizebox{0.4\textwidth}{!}{
% \begin{tabular}{l| c}
% \toprule
% \textbf{Settings} & \textbf{Stage 2} \\
% \midrule
% Batch size & 128 \\
% Learning rate & 2e-5 \\
% Learning schedule & Cosine decay \\
% Warmup ratio & 0.03 \\
% Weight decay & 0 \\
% Epoch & 1 \\
% Optimizer & AdamW \\
% DeepSpeed stage & 3 \\
% Max token & 2048 \\
% \bottomrule
% \end{tabular}}
% \label{tab:train_detail}
% \end{wraptable}

\subsection{Benchmarks}
\label{app_sec:benchmarks}

\begin{itemize}[leftmargin=*]
\item MME~\citep{fu2023mme} is a comprehensive benchmark for evaluating the performance of MLLMs in multi-modal tasks. It measures models' capabilities across two key areas: perception and cognition, using 14 specially designed subtasks that test interpretative and analytical skills.

\item MMBench~\citep{liu2023mmbench} employs a dual approach: it provides an extensive dataset that broadens the range and variety of evaluation questions, and introduces the innovative CircularEval strategy, which uses ChatGPT to convert free-form predictions into structured choices. MMBench-CN is the Chinese version of the benchmark.

\item ScienceQA~\citep{lu2022learn} is a multi-modal benchmark aimed at assessing and diagnosing AI systems' multi-hop reasoning and interpretability in the science domain. It includes a dataset of around 21K multiple-choice questions across various scientific topics, complete with detailed answer annotations, related lectures, and explanations.

\item GQA~\citep{hudson2019gqa} is a dataset designed for advanced visual reasoning in real-world scenarios, using scene graph-based structures to generate 22 million diverse, semantically-programmed questions. It features a novel set of evaluation metrics focused on consistency, grounding, and plausibility, setting a high standard for vision-language task assessment.

\item POPE~\citep{li-etal-2023-evaluating} is an evaluation method for examining object hallucination in MLLMs. It transforms the evaluation into a binary classification task, asking MLLMs simple Yes-or-No questions to identify hallucinated objects. POPE employs various object sampling strategies to reveal model tendencies towards hallucination.

\item VQA V2~\citep{balanced_vqa_v2} evaluates the model’s visual perception capabilities through open-ended questions. It consists of 265,016 images, covering a wide variety of real-world scenes and objects, providing rich visual contexts for the questions. For each question, there are 10 ground truth answers provided by human annotators, which allows for a comprehensive evaluation of the performance of different models in answering the questions accurately.

\item TextVQA~\citep{singh2019towards} focuses on the comprehensive integration of diverse text information within images. It meticulously evaluates the model’s text understanding and reasoning abilities through a series of visual question-answering tasks with rich textual information. Models need to not only understand the visual content of the images but also be able to read and reason about the text within the images to answer the questions accurately.

\item OCRBench~\citep{liu2024ocrbench} is a comprehensive benchmark for evaluating the OCR capabilities of multi-modal language models across five key tasks: text recognition, scene text-centric and document-oriented VQA, key information extraction, and handwritten mathematical expression recognition.

\end{itemize}

\subsection{Overview of the Baselines}
\label{app_sec:baselines}
\subsubsection{General MLLMs}
\label{app_sec:general_mllms}
\begin{itemize}[leftmargin=*]

\item BLIP-2~\citep{li2023blip} is a vision-language pretraining framework that efficiently combines frozen image encoders with large language models (LLMs). It adopts a two-stage training strategy leveraging a lightweight Querying Transformer to bridge the vision-language modality gap, enabling compute-efficient, zero-shot image-to-text generation aligned with natural language instructions.

\item InstructBLIP~\citep{instruct-blip} builds on BLIP-2 by introducing instruction tuning and an instruction-aware Query Transformer. This enhances the model’s ability to extract features for a wide range of vision-language tasks. It achieved state-of-the-art zero-shot performance across 13 benchmarks and demonstrated strong results on fine-tuned tasks such as ScienceQA.

\item IDEFICS~\citep{laurenccon2023obelics} is an open-access vision-language model based on the Flamingo~\citep{alayrac2022flamingo} architecture. Available in both base and instruction-tuned variants (9B and 80B parameters), IDEFICS is trained entirely on publicly available data and models, promoting transparency and accessibility.

\item Qwen-VL \& Qwen-VL-Chat~\citep{Bai:Qwen-VL} extend the Qwen-LM~\cite{bai2023qwen} foundation with a visual encoder and a specialized input-output interface. Through a three-stage training pipeline and a rich multilingual, multimodal corpus, the models achieve strong capabilities in grounding and optical character recognition (OCR) tasks.

\item SPHINX~\citep{lin2023sphinx} is a multimodal large language model that applies joint mixing across model weights, tuning objectives, visual embeddings, and image scales. By unfreezing the LLM during pretraining and integrating diverse instructional and visual signals, SPHINX demonstrates strong performance in fine-grained vision-language understanding and reasoning tasks.

\item mPLUG-Owl2~\citep{ye2024mplug} features a modular architecture with a language decoder interface for unified modality coordination. It employs shared cross-modal modules alongside modality-adaptive components to enhance feature retention and generalization in both unimodal and multimodal settings.

\item Video-LLaVA~\citep{lin2023video} extends multimodal language modeling to unified video and image understanding. By aligning visual modalities into a shared language feature space prior to projection, the model enables effective joint training across visual domains, achieving state-of-the-art results on diverse video and image benchmarks without requiring paired video-image data.

\item LLaVA-v1.5~\citep{liu2024improved} improves upon the original LLaVA~\citep{liu2023llava} model through enhanced visual instruction tuning. Utilizing a CLIP-ViT-L-336px visual backbone and MLP-based projection, it achieves strong performance with high data efficiency. Trained on only 1.2 million publicly available images, it excels at academic-oriented VQA tasks using straightforward prompting strategies.

\end{itemize}

\subsubsection{Training-aware Token Compression Methods}
\begin{itemize}[leftmargin=*]

\item Average Pooling, inspired by the token merging strategy in Qwen2-VL~\citep{wang2024qwen2}, merges every four adjacent visual patches into a single token, effectively reducing the number of visual tokens.

\item QT-LLaVA replaces the original modality projector in LLaVA with a Q-former~\citep{li2023blip} that uses learnable queries to project the input token sequence into a shorter sequence. In this work, we fix to retain $256$ visual tokens.

\item MQT-LLaVA~\citep{hu2024matryoshka} proposes a flexible Query Transformer that enables encoding images into a variable number of visual tokens (up to a predefined maximum), allowing dynamic adaptation to different tasks and computational budgets.

\item LLaMA-VID~\citep{li2024llama} compresses both the instruction and the image into a single token each, resulting in just two tokens per image. This design enables efficient understanding of longer video sequences.

\item VoCo-LLaMA~\citep{ye2024voco} utilizes language models to compress all vision tokens, significantly improving computational efficiency while maintaining multimodal understanding.

\item TokenPacker~\citep{li2024tokenpacker} introduces a visual projector that adopts a coarse-to-fine strategy to reduce the number of visual tokens by up to $80\%$, substantially decreasing the computational cost.

\item LLaVA-Mini~\citep{zhang2025llava} employs a modality pre-fusion module, constructed with Transformer blocks, to integrate visual information into the text tokens in advance. This approach reduces the number of visual tokens fed into the language model, thereby lowering computational cost.

\end{itemize}

\section{Limitations and Future Works}
\label{app_sec:future_works}
As discussed in the Appendix~\ref{app_sec:training_details}, our proposed progressive consistency distillation framework was only applied during the visual instruction tuning phase when training the multi-modal large language models (MLLMs). While this approach has already demonstrated remarkable effectiveness in terms of both model performance and training efficiency, several promising research directions remain unexplored.
For instance, how might our method perform if applied during the model's pretraining stage? Specifically, implementing our framework during projector pretraining, rather than initializing it with publicly available projector weights, could potentially yield greater performance improvements. We hypothesize that this would allow the modality projector's parameter space to adapt to token compression-induced feature space perturbations prior to fine-tuning, thereby further facilitating subsequent supervised fine-tuning.
In future work, we will systematically investigate extending our framework to all training stages (including pretraining and supervised finetuning), with rigorous analysis of both the performance gains and impacts on training-time efficiency.

\section{Broader Impacts}
\label{app_sec:broader_impact}
In this work, we propose a framework to develop efficient multi-modal large language models (MLLMs) through Progressive Consistency Distillation, which can be seamlessly integrated with various token compression strategies without modifying the original model architecture. Our approach demonstrates strong effectiveness, robustness, and generalizability.
On one hand, this contributes to improving the efficiency of MLLMs, thereby facilitating their deployment and practical applications at a societal level, especially for resource-constrained edge devices. 
On the other hand, after training under our framework with token-compressed inputs, the resulting model parameters shift from the optimum of the vanilla model (trained on full inputs) to a new optimum suited for compressed inputs. However, we do not further align these models with human preferences after token compression-based training. While the models perform well across a range of multi-modal tasks, they may carry potential risks of adversarial vulnerabilities or undesirable outputs.

\end{document}